\documentclass[11pt]{article}
\usepackage[utf8]{inputenc}
\usepackage[a4paper, total={6in, 8.2in}]{geometry}
\usepackage{amsmath}
\usepackage{mathtools}
\usepackage{amsthm}
\usepackage{amssymb}
\usepackage[english]{ babel}
\usepackage{dirtytalk}
\usepackage{graphicx}
\usepackage{authblk}
\usepackage[hyphens]{url}
\usepackage{hyperref}
\usepackage{caption}
\usepackage{subcaption}
\usepackage{cases}
\usepackage{xcolor} 
\usepackage{bm}
\usepackage{tikz}
\usetikzlibrary{shapes,arrows.meta,positioning}

\usepackage{tabularx}
\usepackage{blkarray}
\usepackage{enumerate}
\usepackage{comment}
\usepackage{algorithm}
\usepackage{algorithmic}
\usepackage{float}
\usepackage{booktabs}
\usepackage{ragged2e}
\newtheorem{theorem}{Remark}
\usepackage{appendix}
\usepackage{multirow}

\usepackage[bottom]{footmisc}
\title{Electricity Price Prediction using Multi-Kernel Gaussian Process Regression combined with Kernel-Based Support Vector Regression}

\author[1]{Abhinav Das}
\author[2]{Stephan Schl\"uter}
\author[3]{Lorenz Schneider}
\affil[1]{Faculty of Mathematics and Economics, Ulm University}
\affil[2]{Institute of Energy Engineering and Energy Economics, Ulm University of Applied Sciences}
\affil[3]{Emlyon Business School, Lyon, France}

\date{}

\begin{document}

\maketitle

\begin{abstract}
\large This paper presents a new hybrid model for predicting German electricity prices. The algorithm is based on a combination of Gaussian Process Regression (GPR) and Support Vector Regression (SVR). Although GPR is a competent model for learning stochastic patterns within data and for interpolation, its performance for out-of-sample data is not very promising. 
By choosing a suitable data-dependent covariance function, we can enhance the performance of GPR for the German hourly power prices being tested. However, since the out-of-sample prediction is dependent on the training data, the prediction is vulnerable to noise and outliers. To overcome this issue, a separate prediction is calculated using SVR, which applies margin-based optimization. This method is advantageous when dealing with non-linear processes and outliers, since only certain necessary points (support vectors) in the training data are responsible for regression. 
The individual predictions are then linearly combined using uniform weights. When tested on historic German power prices, this approach outperforms the publicly available benchmarks, namely the LASSO estimated autoregressive regression model, deep neural network provided in the recent research by \cite{LAGO2021116983} .
\end{abstract}

\renewcommand{\thefootnote}{}

\vfill

\section{Introduction}
\label{S:Introduction}

Predicting energy prices, particularly electricity prices, has been a challenging research problem ever since the European electricity markets were liberalized. Electricity differs from other energy commodities, such as oil or gas, for several reasons, including non-storability and -- more recently -- weather dependency. Modeling features such as hourly electricity prices, supply/demand, and production has become increasingly difficult because of the rising share of renewable power in total energy production. As a consequence market volatility has increased significantly due to its growing dependence on weather conditions that are, by nature, intermittent \cite{husin2021critical}. Multiple additional factors, such as  non-stationarity of the data, complex inter-dependencies, temporal correlations, for example, complicate price prediction. However, Comparatively accurate electricity price prediction is highly desired by various stakeholders including consumers, energy production facilities, and policy makers \cite{wang2019daily}. Thereby, when choosing a forecasting model, we face the following dilemma: On the one hand, constructing a model that efficiently incorporates complex factors and that provides relatively accurate predictions often comes with high computational costs, due to the model's complexity. On the other hand, simplifying the model to reduce computational demands can lead to a loss in prediction accuracy. The challenge is therefore to develop a balanced model that is both simple and interpretable, but that can still provide sufficiently accurate predictions. Achieving this balance is a key task in model development, and we illustrate this challenge using the German power market as an example. The share of renewable energy in Germany is significantly higher than in other markets, including the French market, for instance. In 2023, for example, the aggregated share from renewables in Germany was above 50\% \cite{Fraunhofer}. Consequently, hourly day-ahead prices are more volatile due to the variability of weather conditions. To model such a volatile commodity, several approaches such as statistical or probabilistic methods \cite{cornell2024probabilistic}, machine learning/deep learning methods \cite{ poggi2023electricity}, or their combination (typically known as a hybrid approach \cite{jiang2023multivariable}) are employed.

Given the challenges, a relatively simple and explainable prediction method that can efficiently model the inter-relationship between the price and its influencing factors is indispensable. To achieve this balance, we have formulated a prediction model that is a linear combination of Gaussian Process Regression (GPR) and Support Vector Regression (SVR). GPR is a probabilistic approach and SVR is a machine learning approach. Both are kernel-based methods that are very good at capturing the non-linear relationship between commodities and their influencing factors. We analyzed various kernel functions for both GPR and SVR that fit the characteristics of German Electricity price data. GPR provides the uncertainty associated with predicted values, while SVR makes point forecasts. We therefore used the conformal prediction approach for the interval associated with the point forecast via SVR. The proposed combined model was implemented using German electricity price data, the predicted residual load data, and the total renewable energy production data for the individual years 2021 to 2023. To analyze the performance of the proposed approach, we compared the predictions from our combined model with the benchmark models proposed by \cite{LAGO2021116983} which are LASSO estimated autoregressive (LEAR) model and the deep neural network (DNN) model. The comparison confirms that the proposed model outperforms the aforementioned benchmark models.

The remainder of the paper is organized as follows. 
In Section \ref{S:RelatedWork}, we give a brief overview of the related work in this field of research. 
In Section \ref{S:ElectricityData}, we describe the electricity price data and factors that affect prices. We also present our data sources and the arrangement of the data. In Section \ref{s:benchmark}, we introduced the benchmark models which we have used for the comparison of performance.
A brief introduction to Gaussian processes is given in Section \ref{S:GaussianProcessRegression}, which also provides a detailed formulation for the GPR and the covariance functions. 
Section \ref{S:SupportVectorRegression} focuses both on the deployment of kernel-based SVR to predict electricity prices, and on the prediction interval for SVR. 
In Section \ref{S:AHybridModel}, we introduce the so-called hybrid model as the linear combination with uniform weights of the predicted prices using SVR and GPR. 
In Section \ref{S:NumericalResults}, we test our hybrid model against chosen benchmarks.
We conclude our work and suggest future avenues for this research in Section \ref{S:Conclusion}. 
Additional mathematical reasoning and explanations are given in the Appendices.


\section{Related Work}
\label{S:RelatedWork}

The prediction of electricity prices has become increasingly important following the liberalization of the European power markets. Approaches such as the autoregressive integrated moving average (ARIMA) model or its extension, the seasonal autoregressive integrated moving average (SARIMA) model, have been widely used to forecast electricity prices. In \cite{contreras2003arima}, the authors use the ARIMA model to predict electricity prices in Spain and California. The model performed with non-uniformity in both markets. Similarly, in \cite{conejo2005day}, the authors proposed a hybrid ARIMA model combined with wavelet transform for short-term electricity price forecasting in the Spanish electricity market. Their results demonstrated the effectiveness of the hybrid model in capturing both short-term fluctuations and long-term trends in electricity prices. A similar study \cite{karabiber2019electricity} combined ARIMA with an artificial neural network (ANN) to forecast electricity prices in the western region of Denmark, highlighting the superiority of ARIMA over the seasonal naive model. In \cite{ghayekhloo2019combination}, the authors explored a combination of machine learning algorithms, including support vector machines (SVMs), ANNs, and decision trees, for electricity price forecasting. Their study showed that ensemble methods combining multiple algorithms often outperform individual models, leading to more accurate predictions.

In recent years, deep learning techniques have gained traction in electricity price prediction. \cite{mehmood2020interval}, for instance, proposed a fuzzy neural network for short-term electricity price forecasting. By incorporating fuzzy logic into neural networks, their model improved prediction accuracy, particularly in capturing uncertainty and non-linearity in electricity price data. Furthermore, in \cite{zhang2020deep}, the authors investigated the use of convolutional neural networks (CNN) for daily electricity price forecasting. Their study demonstrated the ability of CNNs to extract non-linear features from historical price data, leading to enhanced prediction performance compared with traditional models. In \cite{bozlak2024optimized}, the authors compared the use of long short-term neural networks and CNNs with a SARIMA  model. 
\cite{uniejewski2024regularization} tested various methods to achieve better prediction using different regularizations, while in \cite{OCONNOR2024101436}, the authors used the LASSO-based auto-regression method to predict electricity prices. In particular, Long Short-Term Memory (LSTM) networks and Deep Neural Networks have demonstrated superior predictive performance, largely due to their ability to capture nonlinear patterns in the data. When considering both predictive accuracy and computational efficiency, these models have outperformed traditional time-series and statistical approaches \cite{LAGO2021116983}. The ensemble model proposed in \cite{LAGO2021116983} comprises a set of four deep neural networks (DNNs), each trained on datasets of varying lengths: 1-year, 2-year, 3-year, and 4-year historical data. This ensemble has been evaluated only against its constituent models and the ensemble LASSO estimated autoRegressive (LEAR) method. For a more comprehensive performance assessment, comparisons with other ensemble learning approaches are necessary. In addition to data-driven approaches, stochastic models have also been employed for electricity price prediction.

Despite the progress made in electricity price prediction, several challenges remain. One is the increasing complexity of the electricity market, which incorporates renewable energy sources and an increasing number of electric vehicles and batteries. These developments introduce additional uncertainty and non-linearity into price data, requiring more sophisticated modeling techniques. In addition, the deregulation of electricity markets has led to increased market volatility and unpredictability, which poses challenges for traditional forecasting methods. As a result, there is a growing need for innovative approaches that can adapt to these changing market dynamics and provide reliable predictions in uncertain environments. 

In conclusion, electricity price prediction is a multifaceted problem that requires a combination of mathematical, statistical, and computational techniques. Despite the significant progress made in this field, ongoing research is needed to address these emerging challenges and to develop more robust and accurate forecasting models. Although time series models like ARIMA, and SARIMA provide a strong foundation for understanding temporal patterns, their assumption of linear relationships restricts their ability to capture complex non-linear dependencies in time series data \cite{hamilton2020time}. Approaches such as generalized autoregressive conditional heteroskedasticity (GARCH) and threshold autoregressive (TAR) models deal with the non-linearity within datasets, however, these models still fall short in terms of overall flexibility and their ability to handle high-frequency data. 
Machine learning and deep learning techniques offer flexibility and scalability for capturing complex relationships in data, but they require large amounts of data to generalize well and are prone to overfitting when applied to noisy, sparse, or volatile data, such as hourly time series \cite{guo2024hybrid}.  
GPR and SVR offer a robust solution as they are non-parametric models that are inherently capable of handling both linear and non-linear relationships. GPR, with its probabilistic framework, can model uncertainty and capture both short-term and long-term dependencies through flexible kernel functions \cite{williams2006gaussian}. SVR's kernel-based approach excels at handling high-dimensional data and is highly robust to outliers and noise \cite{smola2004tutorial}. Combining these two methods, in an ensemble approach, allows for better generalization in high-frequency time series data, making them a superior choice compared with traditional time series models and deep learning techniques, especially when the dataset is noisy, sparse, or irregular.


\section{Electricity Data}
\label{S:ElectricityData}

In 2023, $57$\% of the total load in Germany was provided by renewable sources, with (onshore) wind being the dominant technology followed by solar and hydropower \cite{Fraunhofer}. This important share, combined with the stochastic nature of these renewable energy sources, suggests that they are a fundamental driver for electricity prices in Germany. Our analysis is based on data provided by the Federal Network Agency of Germany (German: Bundesnetzagentur$^1$\footnote{{$^1$\href{https://www.smard.de/home}{Link to Data: Accessed on March 12, 2025}}}). 
The data include historic prices, residual load forecast, and total renewable energy production forecast. As a general practice in the electricity market, electricity produced from renewable sources is traded with guarantee. Gross demand minus the hourly production from renewable sources, known as the \textit{residual load}, thus becomes more relevant as it determines the amount of additional energy that needs to be supplied by non-renewable energy sources. A high residual load indicates that the energy generated from renewable sources is low and that non-renewable sources are required to satisfy demand \cite{liebl2013modeling}. In this context, Figure \eqref{price_load_scatter} shows the relationship between electricity prices and the residual load forecast, while Figure \eqref{price_ren_scatter} shows the relationship between electricity prices and the total renewable energy production forecast.
Explicitly providing the model with information about residual load and total electricity produced by renewable sources helps the algorithm to learn underlying price patterns. As shown in Figures \eqref{real and transf}-\eqref{real and transf ren_prod}, the data display certain levels of heteroscedasticity, noise, and extreme values, which affect the modeling. A log transformation of the data and normalization help to smooth these properties.
\begin{figure}[ht!]
    \centering
    \begin{subfigure}{0.5\linewidth}
        \includegraphics[height = 5cm, width = 6.5cm]{ 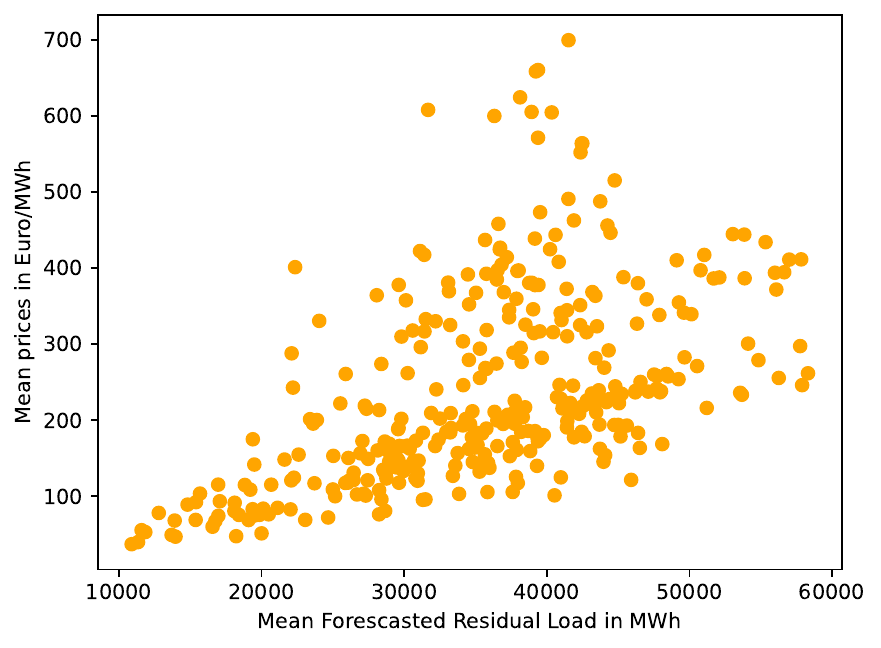}
        \caption{}
        \label{price_load_scatter}
    \end{subfigure}
    \begin{subfigure}{0.45\linewidth}
        \includegraphics[height = 5cm, width = 6.5cm]{ 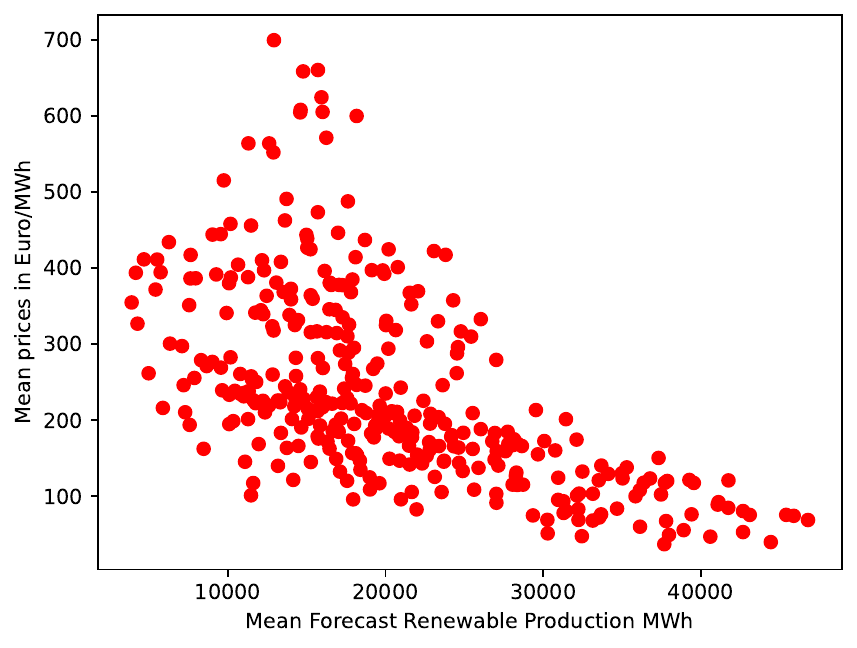}
        \caption{}
        \label{price_ren_scatter}
    \end{subfigure}
    \caption{\centering{Scatter Plot of Historic Data for (a) Daily Average of Forecast Residual Load vs Daily Average of Price and (b) Daily Average of Forecast Total Renewable Production vs Daily Average of Price, for One Year}}
    \label{scatter}
\end{figure} 

\begin{figure}[ht!]
    \centering
    \begin{subfigure}{0.5\linewidth}
        \includegraphics[height = 5cm, width = 6.5cm]{ 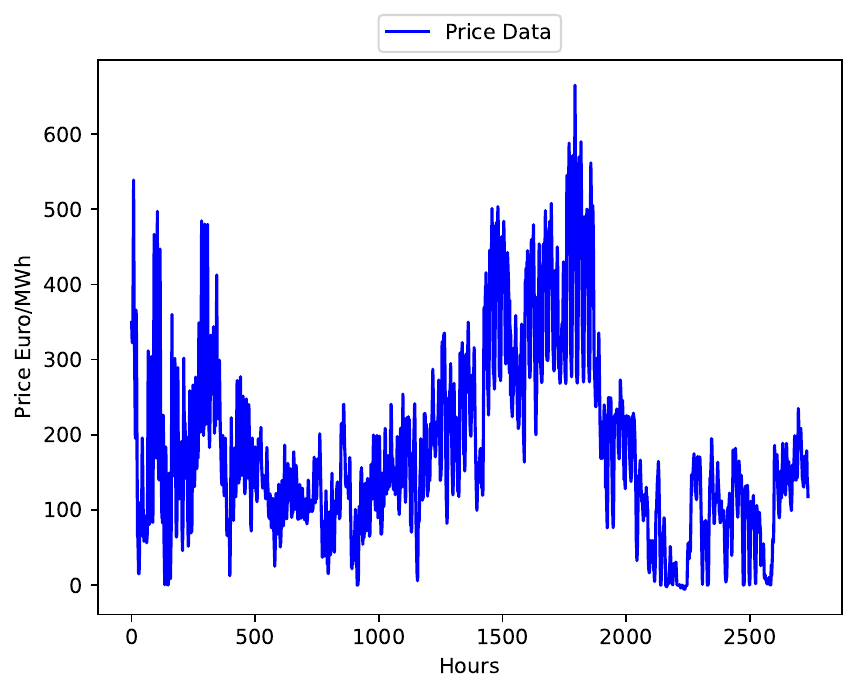}
        \caption{}
    \end{subfigure}
    \begin{subfigure}{0.45\linewidth}
        \includegraphics[height = 5cm, width = 6.5cm]{ 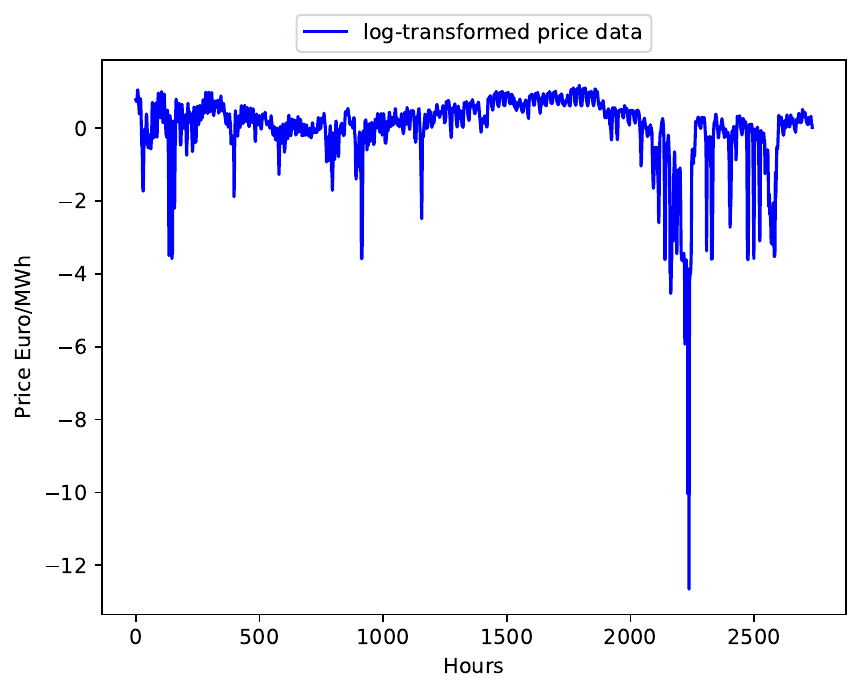}
        \caption{}
    \end{subfigure}
    \caption{Price Data from Sep 30, 2022 to Jan 21, 2023 (a) Real Data and (b) Log-transformed Data}
    \label{real and transf}
\end{figure}

\begin{figure}[ht!]
    \centering
    \begin{subfigure}{0.5\linewidth}
        \includegraphics[height = 5cm, width = 6.5cm]{ 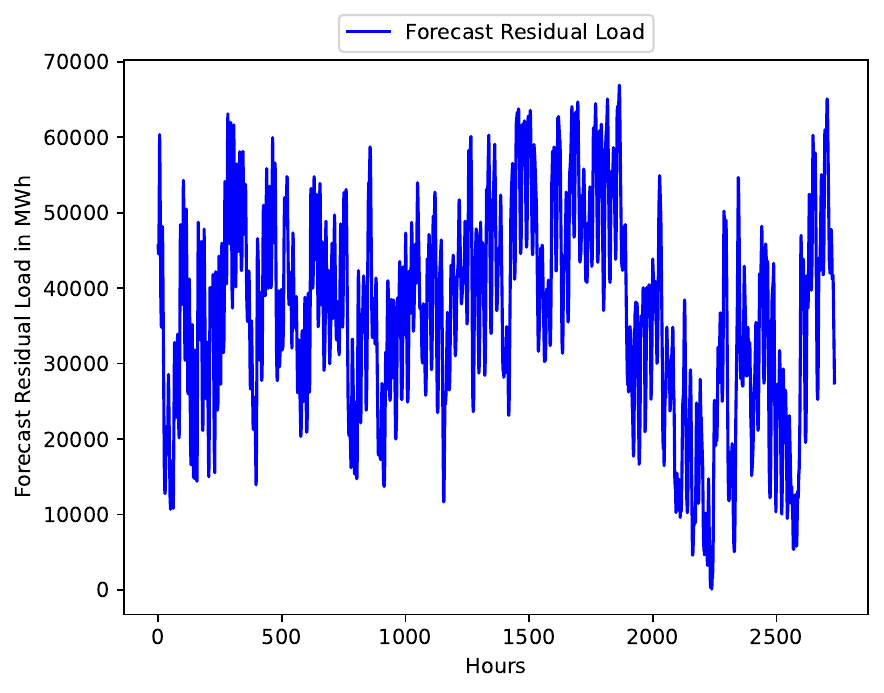}
        \caption{}
    \end{subfigure}
    \begin{subfigure}{0.45\linewidth}
        \includegraphics[height = 5cm, width = 6.5cm]{ 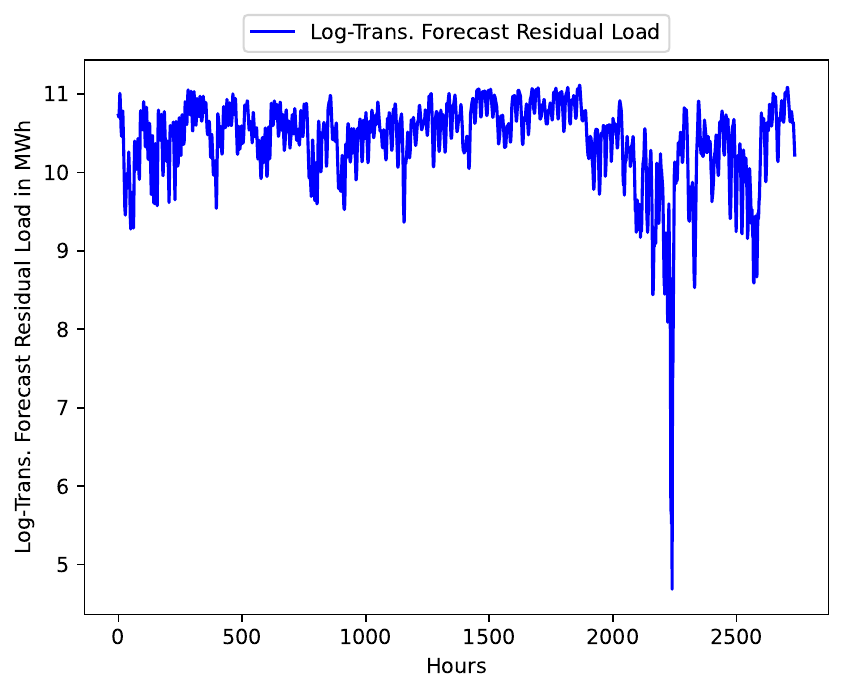}
        \caption{}
    \end{subfigure}
    \caption{Forecast Residual Load Data from Sep 30, 2022 to Jan 21, 2023 (a) Real Data and (b) Log-transformed Data}
    \label{real and transf load}
\end{figure}

\begin{figure}[ht!]
    \centering
    \begin{subfigure}{0.5\linewidth}
        \includegraphics[height = 5cm, width = 6.5cm]{ 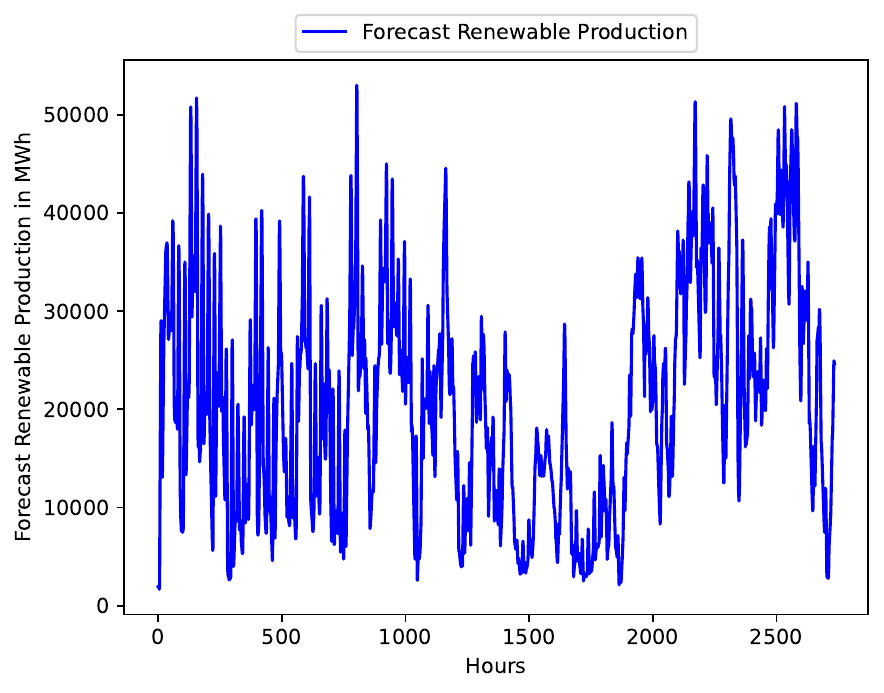}
        \caption{}
    \end{subfigure}
    \begin{subfigure}{0.45\linewidth}
        \includegraphics[height = 5cm, width = 6.5cm]{ 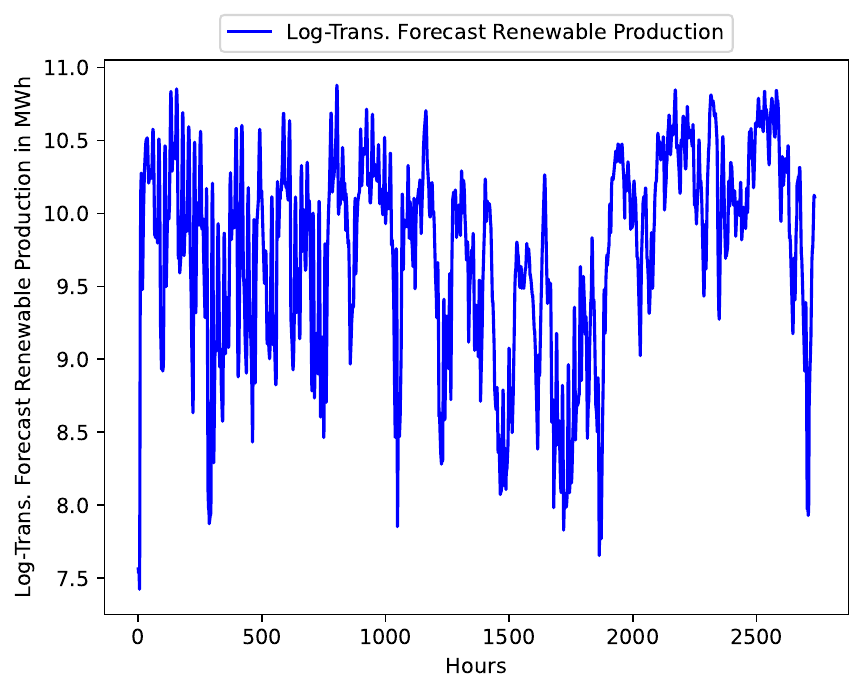}
        \caption{}
    \end{subfigure}
    \caption{Forecast Renewable Energy Production Data from Sep 30, 2022 to Jan 21, 2023 (a) Real Data and (b) Log-transformed Data}
    \label{real and transf ren_prod}
\end{figure}

Weather data are not considered, despite their relevance, since weather is local and station-based, and our model deals with aggregated data for the whole of Germany. Although weather data  are available, it is difficult to efficiently integrate these data into our model, and certain features of the data raise data quality issues. Finally, weather does not affect electricity prices directly, but only indirectly through renewable energy production. This production is considered in our model through solar, wind, biomass, and hydropower. 
We have arranged the datasets as follows:
We consider hourly data for the electricity price, the forecast residual load, and the forecast total renewable energy (day ahead), which means that for one day we have $24$ data points for each variable. For each of these variables, we stack each day in a single vector, which will be a time series of the price, load, and energy production indexed on an hourly basis for each day as follows:
$$\text{Price data} = \begin{bmatrix}
    P_{1}^{(1)}&P_{2}^{(1)}, \cdots, &P_{24}^{(1)}\\
    P_{1}^{(2)}&P_{2}^{(2)}, \cdots, &P_{24}^{(2)}\\
    \vdots & \vdots  &\vdots\\
    P_{1}^{(n)}&P_{2}^{(n)}, \cdots, &P_{24}^{(n)}
\end{bmatrix}_{n\times 24}$$ 

Similarly, we have the load data and total energy production (TRP) data as follows:

$$\text{Load data} = \begin{bmatrix}
    L_{1}^{(1)}&L_{2}^{(1)}, \cdots, &L_{24}^{(1)}\\
    L_{1}^{(2)}&L_{2}^{(2)}, \cdots, &L_{24}^{(2)}\\
    \vdots & \vdots  &\vdots\\
    L_{1}^{(n)}&L_{2}^{(n)}, \cdots, &L_{24}^{(n)}
\end{bmatrix}_{n\times 24}\text{ and\hspace{0.5cm}} \text{TRP data} = \begin{bmatrix}
    R_{1}^{(1)}&R_{2}^{(1)}, \cdots, &R_{24}^{(1)}\\
    R_{1}^{(2)}&R_{2}^{(2)}, \cdots, &R_{24}^{(2)}\\
    \vdots & \vdots  &\vdots\\
    R_{1}^{(n)}&R_{2}^{(n)}, \cdots, &R_{24}^{(n)}
\end{bmatrix}_{n\times 24}$$ 


\section{Benchmark Models}\label{s:benchmark}
The research in \cite{LAGO2021116983} suggest that the LASSO estimated autoregressive (LEAR) model and the  deep neural network (DNN) model are the better performing model when compared to different class and categories of the models for electricity price forecasting. In this study we compare the performance of our hybrid model which predicts via GPR and SVR with both of these benchmark models. We also compare the individual predictions by GPR and SVR with both of these benchmarks. In this section we briefly introduce the construction of LEAR and DNN Models and for detail study of these we request the readers to refer to \cite{LAGO2021116983}:

\subsection{LEAR Model}\label{ss:lear_model}
Let us assume that we have a data set with price of electricity , forecast residual load and forecast total renewable energy production in hourly resolution. With the assumption that the current price of the electricity is dependent on the past prices, forecast residual load and forecast total renewable energy production, the LEAR model to predict the electricity price on $h^{th}$ hour of day $i$, is as follows:
\begin{equation}\label{lear}
    \begin{split}
        P^{(i)}_{h} = & f(P^{(i-1)}, P^{(i-2)}, P^{(i-3)}, P^{(i-7)}, L^{(i)}, L^{(i-1)}, L^{(i-7)}, R^{(i)}, R^{(i-1)}, R^{(i-7)}, \boldsymbol{\theta}_{h})+\epsilon^{(i)}_{h}\\ 
         = & \sum_{j=1}^{24}\theta_{h,j}P^{(i-1)}_{j} + \sum_{j=1}^{24}\theta_{h,24+j}P^{(i-2)}_{j} + \sum_{j=1}^{24}\theta_{h,48+j}P^{(i-3)}_{j}\\
         & + \sum_{j=1}^{24}\theta_{h,72+j}P^{(i-7)}_{j} + \sum_{j=1}^{24}\theta_{h,96+j}L^{(i)}_{j} + \sum_{j=1}^{24}\theta_{h,120+j}L^{(i-1)}_{j} \\
         & + \sum_{j=1}^{24}\theta_{h,144+j}L^{(i-7)}_{j} + \sum_{j=1}^{24}\theta_{h,168+j}R^{(i)}_{j} + \sum_{j=1}^{24}\theta_{h,192+j}R^{(i-1)}_{j}\\
         &  + \sum_{j=1}^{24}\theta_{h,216+j}R^{(i-7)}_{j} + + \sum_{j=1}^{24}\theta_{h,240+j}z^{(i)}_{j} +\epsilon^{(i)}_{h}
    \end{split}
\end{equation}
where $\boldsymbol{\theta}_{h} = [\theta_{h,1},\cdots, \theta_{h,247}]$  are the parameters for LEAR. These parameters are estimated using LASSO as follows:
$$\boldsymbol{\hat{\theta}}_{h} = \operatorname*{argmin}_{\boldsymbol{\theta}_{h}} \sum_{i=8}^{N}(P^{(i)}_{h} - \hat{P}^{(i)}_{h})^{2} + \lambda \sum_{k = 1}^{247}\vert\theta_{h,k}\vert$$
where $\lambda \geq 0$ is the regularization parameter controlling the sparsity of the solution.

\subsection{DNN}\label{ss:dnn}
A deep neural network with 4 hidden layers, trained under a multivariate framework was considered. The parameters $\theta$ are optimized using the Adam algorithm, while hyperparameters and input features were selected via Tree-structured Parzen Estimator (TPE) which is a Bayesian optimization method. This DNN model can also be accessed as python library$^2$.\footnote{$^2$\href{https://epftoolbox.readthedocs.io/en/latest/modules/started.html}{Link to Python Toolbox}}


\section{Gaussian Process Regression}
\label{S:GaussianProcessRegression}

\subsection{Gaussian Processes}
\label{SS:GaussianProcess}

A stochastic process given by $G$ = $\{P_{\mathbf{t}}: \mathbf{t}\in T\}$, where $T$ is the index set, is said to be a Gaussian process (GP) if and only if for every finite set of indices 
$\{\mathbf{t}_{1},\cdots, \mathbf{t}_{n}\}$ in $T$ the random vector, say, 
$\mathbf{P} = \{P_{\mathbf{t}_{1}}, \cdots, P_{\mathbf{t}_{n}}\}$ is jointly Gaussian, i.e., the joint density is given by:
\begin{equation}\label{Gaussian density}
    f_{\mathbf P} = \frac{\exp\left(\text{\textminus}\frac{1}{2} \left({\mathbf p} \text{\textminus} {\boldsymbol\mu}_{\mathbf{P}}\right)^\mathrm{T}{\boldsymbol\Sigma}_{\mathbf{P}}^{\text{\textminus}1}\left({\mathbf p}\text{\textminus}{\boldsymbol\mu}_{\mathbf{P}}\right)\right)}{\sqrt{(2\pi)^k |\boldsymbol\Sigma_{\mathbf{P}}|}},
\end{equation} 
where $\boldsymbol{\mu}_{\mathbf{P}}$ is the mean vector of $\mathbf{P}$ given by the mean function $\boldsymbol{m}(x)$, and $\boldsymbol{\Sigma}_{\mathbf{P}}$ is the covariance matrix given by the covariance function $\bar{K}(\mathbf{t}_{i}, \mathbf{t}_{j})$. These parameters fully explain the Gaussian process \cite{GP_vag002, pavliotis2014stochastic}. In this study, $T\subset\mathbf{R}^{248}$ such that $T=\{\mathbf{t}_{i}: \mathbf{t}_{i}, i = 1,\cdots,n\}$
where $$\mathbf{t}_{i} = [i, P^{(i-1)}, P^{(i-2)}, P^{(i-3)}, P^{(i-7)}, L^{(i)}, L^{(i-1)}, L^{(i-7)}, R^{(i)}, R^{(i-1)}, R^{(i-7)}, BD^{(i)}],$$ $P^{(v)},L^{(v)},R^{(v)} \in \mathbf{t}_{i}$ are deterministic and $BD$ represents the binary dummy variable for the week days.
\hfill\newline\newline
The covariance function $K:T\times T\rightarrow \mathbb{R}$  must be positive, semi-definite, and symmetric. If $\boldsymbol{\Sigma}_{\mathbf{P}}$ is the covariance matrix formed by $K$, then it follows that: 
\begin{eqnarray*}
\mathbf{x}^{\top} \boldsymbol{\Sigma}_{\mathbf{P}} \mathbf{x} \geq 0\hspace{0.2cm} \text{for any $\mathbf{x}\in \mathbb{R}^{\text{n}}$ }\hspace{0.2cm} \text{and $a_{i,j} = a_{j,i}\forall a_{i,j} \in \boldsymbol{\Sigma}_{\mathbf{P}}$}.
\end{eqnarray*}
The mean function is simply defined as $\boldsymbol{m}(x): T\rightarrow \mathbb{R}$.
\hfill\newline\newline
Covariance functions are often referred to as kernels. In general, a kernel is a bivariate function that is used to transform a function $g$ via a convolution operation $C$:
\begin{equation*}
(Cg)(\mathbf{u}) = \int_{T}g(\mathbf{v})K(\mathbf{u},\mathbf{v})d\mathbf{v},
\end{equation*}
where $K(\mathbf{u},\mathbf{v})$ is the kernel that defines how the function is transformed \cite{debnath2016integral}. Analogously, the covariance function specifies the Gaussian process. Hereafter, the terms covariance functions and kernel are used interchangeably.


\subsection{Gaussian Process Regression}
\label{SS:GaussianReg}

Classified as a non-parametric Bayesian method \cite{orbanz2010bayesian}, GPR is a simple yet powerful method for modeling complex relationships in a dataset. GPR can be viewed both in terms of standard regression, where the output is the linear combination of the input variables, and as a functional form where the Gaussian process denotes the distribution over functions. 
The latter is often referred to as the functional space view \cite{williams2006gaussian}. In the functional space view, the prior distribution is defined by the mean and covariance function, $\mathbf{P}\sim \mathcal{GP}\left(\boldsymbol{\mu}_{P}, \boldsymbol{\Sigma}_{P}\right)$. 
Using the Bayesian method, the posterior distribution is then obtained using the prior and observed data. Since a Gaussian process is fully specified by its mean and covariance functions, with a given dataset, we only need to estimate the mean vector and the covariance matrix. 
In practical situations, it is common to assume the mean function as a constant function or zero (this can be achieved by normalization), allowing the Gaussian process to be fully specified by the covariance function, which reduces the complexity. The choice of covariance function therefore depends on the nature of the data, and the parameters of this covariance function (also known as the hyperparameters of the Gaussian process) are then estimated based on the dataset. In terms of machine learning, this whole process is called training, which includes estimating the hyperparameter and deriving the covariance matrix using the chosen function. This step enables us to define the prior and obtain its parameters. 

Predicting the value of the random variable (function) at unobserved points, say $\mathbf{t}^{*}$, is achieved by computing the posterior distribution.
Let us assume that we have a finite collection 
$\mathbf{P}=\{P_{\mathbf{t}_{1}},\cdots, P_{\mathbf{t}_{n}}\}$ 
from a Gaussian process, observed at points 
$T = \{\mathbf{t}_{1},\cdots,\mathbf{t}_{n}\}$ 
with a given covariance function $K_{\text{GPR}}$ and mean function $\boldsymbol{\mu}_{P}$. 
Given previous observations, we intend to predict at $u$ unobserved points, say 
$T^{*} = \{\mathbf{t}_{1}^{*}, \cdots, \mathbf{t}_{u}^{*}\}$, 
which is denoted by
$\mathbf{P}^{*} = \{P_{\mathbf{t}_{1}^{*}}, \cdots, P_{\mathbf{t}_{u}^{*}} \}$. 
Assuming that $\mathbf{P} \text{ and } \mathbf{P}^{*}$ are jointly Gaussian, their joint distribution is as follows:

\begin{equation}
    \begin{bmatrix}
        \mathbf{P}\\
        \mathbf{P}^{*}
    \end{bmatrix}
    \sim \mathcal{N} 
    \begin{pmatrix}
        \begin{bmatrix}
            \boldsymbol{\mu}_{\mathbf{P}}\\
            \boldsymbol{\mu}_{\mathbf{P^{*}}}
        \end{bmatrix}, & \begin{bmatrix}
        \Sigma_{\mathbf{P}}+\sigma_{n}\mathbf{I} & \Sigma_{\mathbf{P},\mathbf{P}^{*}}\\
        \Sigma_{\mathbf{P}^{*},\mathbf{P}} & \Sigma_{\mathbf{P}^{*}}
    \end{bmatrix}_{(n+u)\times(n+u)}    
    \end{pmatrix}
\end{equation} 
where, \begin{itemize}
    \item $\boldsymbol{\mu}_{\mathbf{P}}$ and $\boldsymbol{\mu}_{\mathbf{P^{*}}}$ are the respective mean vectors for $\mathbf{P}$ and $\mathbf{P^{*}}$
    \item $\Sigma_{\mathbf{P}} = \{K_{\text{GPR}}(\mathbf{x}_{i},\mathbf{x}_{j})\}_{i,j=1}^{n}$ is a covariance matrix from $\mathbf{P} \text{ and }\mathbf{P}$ of order $n \times n$,
    \item $\Sigma_{\mathbf{P},\mathbf{P}^{*}} = \{K_{\text{GPR}}(\mathbf{x}_{i},\mathbf{x}^{*}_{j})\}_{i=1,\cdots,n,j=1,\cdots,m}$ is a covariance matrix from $\mathbf{P} \text{ and }\mathbf{P}^{*}$ of order $n \times u$,
    \item $\Sigma_{\mathbf{P}^{*},\mathbf{P}} = \left(\Sigma_{\mathbf{P},\mathbf{P}^{*}}\right)^{\top}$ 
    \item $\Sigma_{\mathbf{P}^{*}} = \{K_{\text{GPR}}(\mathbf{t}^{*}_{i},\mathbf{t}^{*}_{j})\}_{i,j=1}^{m}$ is a covariance matrix from $\mathbf{P}^{*}\text{ and }\mathbf{P}^{*}$ of order $u \times u$,
    \item $\mathbf{I}$ is the identity matrix of size $n\times n$ and $\sigma_{n}$ is the noise variance.
\end{itemize}
As shown in \cite{williams2006gaussian}, the respective posterior mean and posterior covariance are given by:
    \begin{align}
        \boldsymbol{\mu}^{*}\label{posterior mean} &= \boldsymbol{\mu}_{\mathbf{P^{*}}}+\Sigma_{\mathbf{P}^{*},\mathbf{P}} \left(\Sigma_{\mathbf{P}}+\sigma_{n}\mathbf{I}\right)^{\text{\textminus}1}\left(\mathbf{P} \text{\textminus} \boldsymbol{\mu}_{\mathbf{P}}\right), \\
        \boldsymbol{\Sigma}^{*}\label{posterior cov} &= \Sigma_{\mathbf{P}^{*}} \text{\textminus} \Sigma_{\mathbf{P}^{*},\mathbf{P}} \left(\Sigma_{\mathbf{P}}+\sigma_{n}\mathbf{I}\right)^{\text{\textminus}1}\Sigma_{\mathbf{P},\mathbf{P}^{*}}.
    \end{align}
Here $\boldsymbol{\mu^{*}}$ and $\boldsymbol{\Sigma}^{*}$ are of order $u\times 1$ and $u\times u$, respectively. As discussed earlier, if we assume that $\boldsymbol{\mu}_{\mathbf{P}}=\boldsymbol{\mu}_{\mathbf{P^{*}}}=0$, equation \eqref{posterior mean} further simplifies to $\boldsymbol{\mu}^{*} = \Sigma_{\mathbf{P}^{*},\mathbf{P}} \left(\Sigma_{\mathbf{P}}+\sigma_{n}\mathbf{I}\right)^{\text{\textminus}1}\mathbf{P} $.
In \cite[Chapter 2]{williams2006gaussian} it has been shown that $\boldsymbol{\mu^{*}}$ is the best estimate of the prediction at the new point $\mathbf{t^{*}}$. 
Similarly, $\boldsymbol{\Sigma}^{*}$ gives the variance of the prediction, which quantifies the uncertainty associated with the prediction. 
The predicted value at the test point $\mathbf{t}_{s}^{*} \in T^{*}$ is therefore $\mu_{s}^{*}$, the $s^{th}$ component of $\boldsymbol{\mu^{*}}$, and the uncertainty associated with the prediction is $\sqrt{\Sigma_{ss}}$. 
The detailed derivation for the posterior mean and variance can be found in  \cite[Chapter 2]{tong2012multivariate}. The prediction interval at the test point $\mathbf{t}_{s}^{*} \in T^{*}$ is given by:
\begin{equation}\label{gpr_interval}
    I_{GPR} = [lb_{GPR},ub_{GPR}], \hspace{0.5cm} \text{such that}
\end{equation}
$$lb_{GPR}=\mu_{s}^{*}\text{\textminus}z_{\alpha/2}\sqrt{\Sigma_{ss}} \text{ and } ub_{GPR} = \mu_{s}^{*}+z_{\alpha/2}\sqrt{\Sigma_{ss}}.$$
Here, $z_{\alpha/2}$ denotes the z-score at a confidence level of $1-\alpha$. We have chosen a confidence interval of 95\%.


\subsubsection{Covariance Functions}
\label{SS:CovarianceFunctions}

Covariance functions describe the similarity of two random functions \cite{wilson2013gaussian} and exist in various different forms, e.g., isotropic stationary, anisotropic stationary, locally stationary, non-stationary, etc. For more details, see \cite{genton2001classes,williams2006gaussian} and the references therein. The chosen covariance function must reflect the structural properties (roughness or smoothness, short- or long-term fluctuations) of the data. The so-called decay rate of the function must also be considered. It determines the speed at which the prediction reverts to the mean of the Gaussian process as the prediction point moves farther away from the observed points (similar to the mean reversion parameter in an Ornstein-Uhlenbeck process). Given the structure of electricity prices, we chose a squared exponential function and a rational quadratic function to construct our model. In the following, we describe both functions and also include sample predictions for illustration. We forecast prices on July 17, 2022. For our prediction, we trained our model with 365 days where each input and output are chosen as explained in the section \ref{SS:GaussianProcess}. 

\begin{enumerate}
    \item Squared Exponential Covariance: This function is also known as the Gaussian covariance function and is appropriate for datasets with local structures. In our case, the electricity price shows repeating patterns in similar situations such as weekends, public holidays, and hours of the day -- in other words, seasonality\cite{ulapane2020hyper}. In \cite{quinonero2005unifying}, the squared exponential covariance function is used to capture the local structure within the dataset. The electricity prices are also locally smooth, which is another motivation for using a squared exponential covariance in this case. It is defined as follows:
    $$K_{se}(\mathbf{t}_{i},\mathbf{t}_{j}) = \sigma_{se}^{2}\exp\left(\text{\textminus}\frac{\vert\vert \mathbf{t}_{i} \text{\textminus}\mathbf{t}_{j} \vert\vert^{2}}{2\ell_{se}^{2}}\right).$$
    Here, $\vert\vert \mathbf{t_{i}} \text{\textminus}\mathbf{t_{j}} \vert\vert$ is the Euclidean distance between points $\mathbf{t}_{i}$ and $\mathbf{t}_{j}$, while $\sigma_{se}$ and $\ell_{se}$ are respectively the variance and length parameters for the $K_{se}$, considered as hyperparameters of the Gaussian process. These hyperparameters are estimated using the maximum likelihood estimation (MLE).
    \hfill\newline
    \item Rational Quadratic Covariance: Electricity price data exhibit short-term variations due to changes in renewable energy production, cf. \cite{du2024integrating}, whereas long-term changes stem from new policy implications, cf. \cite{adom2017long}. In such a situation, the rational quadratic function is an appropriate choice. It is defined as follows:
    $$K_{rq}(\mathbf{t}_{i},\mathbf{t}_{j}) = \sigma_{rq}^{2}\left(1+\frac{\vert\vert\mathbf{t}_{i} \text{\textminus} \mathbf{t}_{j}\vert\vert^{2}}{2\alpha \ell_{rq}^{2}}\right)^{\text{\textminus}\alpha_{rq}}$$
    where $\mathbf{t_{i}}$ and $\mathbf{t_{j}}$ are $d-$dimensional time points.
\end{enumerate}
The assumption of smoothness in squared exponential covariance function, governed by the single length-scale parameter $\ell_{se}$, restricts the model to accurately represent functions showing the localized irregularities across different scales. This limitation is eliminated in case of rational quadratic covariance function since it is a scale mixture of squared exponential covariance function with varying scale parameters. The mixture results in a function which has polynomial decay, controlled by the shape parameter $\alpha$ allowing it to capture the multi-scale phenomena or heavier-tailed dependencies. However, in a scenario where the underlying process is smooth across the input space, the rational quadratic function allow the Gaussian process to learns overs the less smooth function which might lead to less accurate model. 

Next we provide the mathematical explanation for the performance of these covariance function when used individually. We also compare the performance of the individual function with their additive combination.
We discuss the local and global behavior of these two individual covariance functions mentioned in the previous section, as well as the behavior of their sum. For this analysis, we use partial differentiation with respect to the distance between any two points in the input space, denoted by $\mathbf{t}_{i}$ and $\mathbf{t}_{j}$. Since squared exponential and rational quadratic covariance functions are isotropic \cite{genton2001classes}, we can therefore consider that $K(\mathbf{t}_{i},\mathbf{t}_{j}) = K(||\mathbf{t}_{i}\text{\textminus}\mathbf{t}_{j}||) $. 
Assuming $r = ||\mathbf{t}_{i} \text{\textminus} \mathbf{t}_{j}||$, we have:
    \begin{equation}\label{se_1}
        \begin{split}
              K_{se}(r) = &\sigma_{se}^{2}\exp\left(\text{\textminus}\frac{ r^{2}}{2\ell_{se}^{2}}\right)\\
             \implies \frac{\partial K_{se}(r)}{\partial r} = &\sigma_{se}^2 \left( \text{\textminus}\frac{r}{\ell_{se}^2} \right) \exp \left( \text{\textminus}\frac{r^2}{2 \ell_{se}^2} \right)\text{ and }\\ 
        \end{split}
    \end{equation}
\begin{equation}\label{se_2}
    \frac{\partial^2 K_{se}(r)}{\partial r^2} = \sigma_{se}^2 \left( \frac{r^2}{\ell_{se}^4} \text{\textminus} \frac{1}{\ell_{se}^2} \right) \exp \left( \text{\textminus}\frac{r^2}{2 \ell_{se}^2} \right).
\end{equation}
Analogously, for the rational quadratic function we obtain:
\begin{equation}\label{rq_1}
    \begin{split}
        \frac{\partial K_{rq}(r)}{\partial r} = \text{\textminus}\frac{r.\sigma_{rq}^2}{ \ell_{rq}^2} \cdot \left( 1 + \frac{r^2}{2 \alpha_{rq} \ell_{rq}^2} \right)^{\text{\textminus}\alpha_{rq} \text{\textminus} 1} \text{ and }
    \end{split}
\end{equation}
\begin{equation}\label{rq_2}
    \begin{split}
        \frac{\partial^2 K_{rq}(r)}{\partial r^2} = &  \frac{\sigma_{rq}^2 . r^2 . (\alpha_{rq} + 1)}{(\alpha_{rq} \ell_{rq}^4)}\cdot  \left( 1 + \frac{r^2}{2 \alpha_{rq} \ell_{rq}^2} \right)^{\text{\textminus}\alpha_{rq} \text{\textminus} 2}\\
        & \text{\textminus}\left(\frac{\sigma_{rq}}{\ell_{rq}}\right)^{2} \cdot \left( 1 + \frac{r^2}{2 \alpha_{rq} \ell_{rq}^2} \right)^{\text{\textminus}\alpha_{rq} \text{\textminus} 1}
    \end{split}
\end{equation}
Similarly, for $K = K_{se} + K_{rq}$ we have: 
\begin{equation}
    \begin{split}\label{K_1}
        \frac{\partial K(r)}{\partial r} = &\frac{\partial K_{se}(r)}{\partial r} + \frac{\partial K_{rq}(r)}{\partial r}\\
         = & \sigma_{se}^2 \left( \text{\textminus}\frac{r}{\ell_{se}^2} \right) \exp \left( \text{\textminus}\frac{r^2}{2 \ell_{se}^2} \right) \text{\textminus} \frac{r.\sigma_{rq}^2}{ \ell_{rq}^2} \cdot \left( 1 + \frac{r^2}{2 \alpha_{rq} \ell_{rq}^2} \right)^{\text{\textminus}\alpha_{rq} \text{\textminus} 1} \text{ and }
    \end{split}
\end{equation}
\begin{equation}\label{K_2}
    \begin{split}
        \frac{\partial^2 K(r)}{\partial r^2} =& \sigma_{se}^2 \left( \frac{r^2}{\ell_{se}^4} \text{\textminus} \frac{1}{\ell_{se}^2} \right) \exp \left( \text{\textminus}\frac{r^2}{2 \ell_{se}^2} \right) \\ 
        &+ \frac{\sigma_{rq}^2 . r^2 . (\alpha_{rq} + 1)}{(\alpha_{rq} \ell_{rq}^4)}\cdot  \left( 1 + \frac{r^2}{2 \alpha_{rq} \ell_{rq}^2} \right)^{\text{\textminus}\alpha_{rq} \text{\textminus} 2} \\
        & \text{\textminus} \left(\frac{\sigma_{rq}}{\ell_{rq}}\right)^{2} \cdot \left( 1 + \frac{r^2}{2 \alpha_{rq} \ell_{rq}^2} \right)^{\text{\textminus}\alpha_{rq} \text{\textminus} 1}.
    \end{split}
\end{equation}

From Equation \eqref{se_1}, it is clear that the change in the squared exponential covariance function with respect to the distance between two points (we refer to this as sensitivity) is very small when the two points in the input space are close to each other. Specifically: $\text{ $\frac{\partial K_{se}(r)}{\partial r} \rightarrow 0 $ as $r\rightarrow 0$}$. For larger distances, the change in the covariance function becomes less significant because of the exponential factor in the RHS of Equation \eqref{se_1}: $\exp(\text{\textminus}\frac{r^{2}}{2\ell_{se}^{2}})$ tends to 0 as $r\rightarrow \infty.$ The length scale in the training data, estimated over the training input and output, also provides information about how quickly or slowly our covariance function decays. From Equation \eqref{se_1}, we can see that when $\ell_{se}$ increases, the decay becomes slower. For the rational quadratic covariance function, the case is similar when two input points are closer because: $\frac{\partial K_{rq}(r)}{\partial r}\rightarrow 0$ as $r\rightarrow 0$. However, for more distant points, the decay is slower due to the polynomial term in the denominator: 
\begin{eqnarray*}
\frac{1}{\left(1+\frac{r^{2}}{2\alpha_{rq}\ell_{rq}^{2}}\right)^{(\alpha_{rq} +1)}}\rightarrow 0 \text{ as } r \rightarrow \infty.
\end{eqnarray*}

 Additionally, in the rational quadratic case, $\alpha$ provides more flexibility for addressing the sensitivity of the rational quadratic function. From Equation \eqref{rq_1}, we observe that the smaller the $\alpha$, the slower the decay. Similarly, from Equation \eqref{se_2}, we see that the change in structure (curvature) is very fast, since the exponential term in the RHS dominates as the distance between the points in the input space increases and converges to zero. It is, nonetheless, smooth. The convergence for the rational quadratic covariance function is comparatively slower than that of the squared exponential due to the presence of polynomial terms. However, the smoothness varies with $\alpha_{rq}$, as can be seen in Equation \eqref{rq_2}.

 \begin{figure}[ht!]
    \centering
    \centering
        \begin{subfigure}{0.5\linewidth}
            \includegraphics[height = 6.5cm, width = 6.5cm]{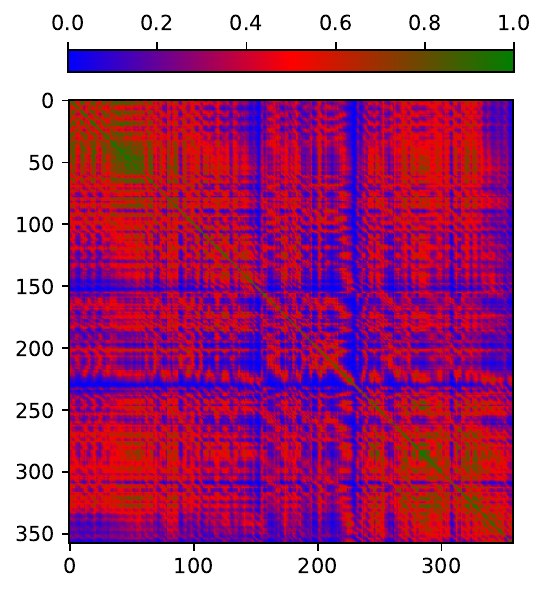}
            \caption{\centering{number of values$<0.2: 28484$} }
            \label{rb_stand}
        \end{subfigure}
        \begin{subfigure}{0.45\linewidth}
            \includegraphics[height = 6.5cm, width = 6.5cm]{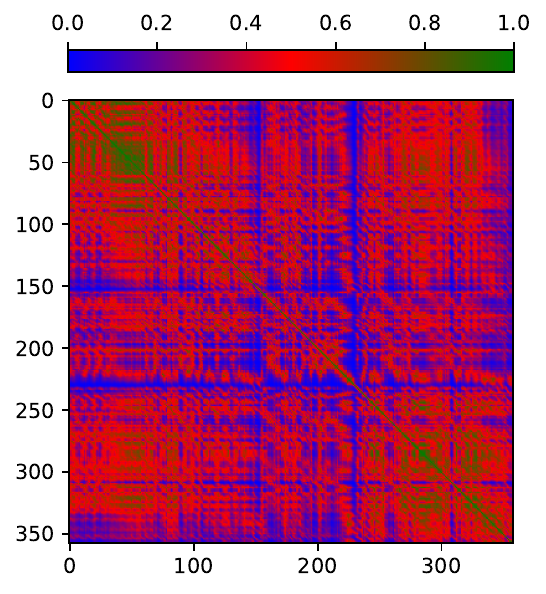}
            \caption{\centering{number of values$<0.2: 22236$} }
            \label{rq_stand}
        \end{subfigure}
        \begin{subfigure}{0.45\linewidth}
            \includegraphics[height = 6.5cm, width = 6.5cm]{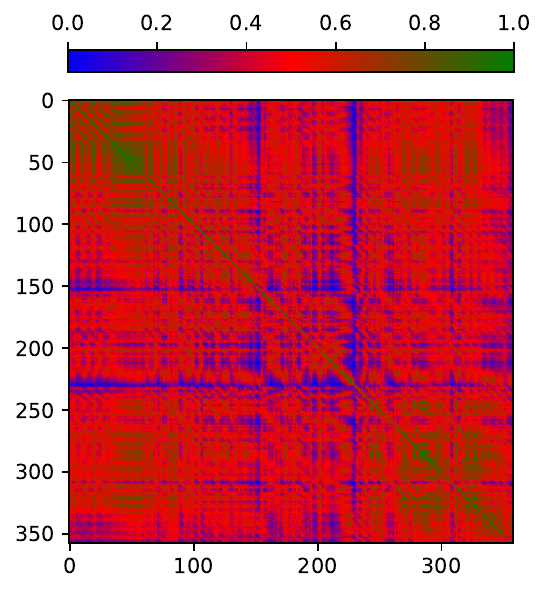}
            \caption{\centering{number of values$<0.2: 4556$} }
            \label{cm_stand}
        \end{subfigure}
        \caption{\centering{Behavior of the Covariance Functions: (a)Via the Squared Exponential, (b) Via the Rational Quadratic, and (c)Via a Combination of Both}}
        \label{covariance_structure}
\end{figure}

 To verify the theoretical findings from equations \eqref{se_1} to \eqref{K_2}, we tested these covariance functions on the GPR inputs. The parameters of the covariance function are estimated using 248-dimensional input, $\mathbf{t}_{i}$ and 1-dimensional output from past 365 days. As an illustration, we took data from July 17, 2021 to July 16, 2022. The inputs are the scaled-hour, log-transformed forecast residual load and the log-transformed forecast total renewable energy production. The training inputs designed as explained in Section \ref{SS:GaussianProcess} where as the the training outputs are the electricity price for a particular hour.

For the purpose of comparison we scaled the the covariance matrices to [0,1] and considered the values less than 0.2 comparatively insignificant. We see from Figure \ref{rb_stand} that the number of elements less than 0.2 in the covariance matrix, obtained using the squared exponential function is highest, which is justified by equations \eqref{se_1} and \eqref{se_2}, inferring that the distant points have near-zero covariance. 
In Figure \ref{rq_stand}, which displays the rational quadratic function, the number of elements less than 0.2 is lesser than that of squared exponential case since the decay is slower, as discussed previously. In terms of rational long-range dependence, the rational quadratic covariance function performs better than the squared exponential covariance function. However, the latter can be considered for modeling purposes because of the smoothness property, for which a squared exponential kernel is preferable. 
Both covariance functions have advantages and drawbacks, and these limitations can be reduced by combining these two functions in a summation. In equations \eqref{K_1} and \eqref{K_2}, and in Figure \ref{cm_stand}, it is evident that the covariance function $K=K_{se}+K_{rq}$ is capable of capturing both smoothness and long-range dependencies. This empirical observation is supported by fundamental properties of kernel function spaces. Let \( \mathcal{H}_{SE} \) and \( \mathcal{H}_{RQ} \) denote the Reproducing Kernel Hilbert Spaces (RKHSs) associated with \( K_{SE} \) and \( K_{RQ} \), respectively. Then the RKHS of the sum kernel satisfies:
\[
\mathcal{H}_K = \mathcal{H}_{SE} + \mathcal{H}_{RQ},
\]
meaning it contains all functions that can be expressed as \( f = f_1 + f_2 \), with \( f_1 \in \mathcal{H}_{SE} \), \( f_2 \in \mathcal{H}_{RQ} \). Since the sum of two RKHSs is strictly larger (in terms of expressiveness) than either component space, this inclusion guarantees a strictly richer hypothesis class one capable of modeling both smooth, localized variations and broader, rougher trends simultaneously.

Moreover, in the Bayesian framework, the Gaussian Process prior induced by this composite kernel assigns non-zero probability to a broader class of functions. This enhances prior support and leads to more robust posterior inference, particularly in heterogeneous data regimes. The model becomes less prone to underfitting, as it does not rely solely on the narrow inductive biases of a single kernel type. Thus, both mathematical structure and empirical performance point unequivocally to the superiority of the additive kernel formulation. It is not simply a hybrid; it is a principled extension with demonstrably greater modeling capacity.

We discuss the summation of Gaussian processes and its effect on the prior and posterior mean and variance in Appendix \ref{A:SumOfGaussianProcesses}. In addition, in Appendix \ref{A:ExploringPeriodicBehavior}, we discuss the ability of $K=K_{se}+K_{rq}$ to capture the periodic behavior of the data. For the uncertainty (posterior variance) associated with the prediction, we can infer from Equation \eqref{posterior cov}, and from the analysis of the covariance function, that the posterior variance for the combined covariance function includes the variance of both covariance functions. This, in turn, leads to a better adjustment of the variance in the input space. There are many other covariance functions, such as the exponential covariance function, the polynomial covariance function, or functions belonging to the maternal family \cite{paciorek2003nonstationary}. However, combining the squared exponential and rational quadratic functions provides an appropriate balance for the required short-range and long-range changes and dependencies.
\begin{figure}[ht!]
    \centering
    \includegraphics[height = 5.2cm, width = 12cm]{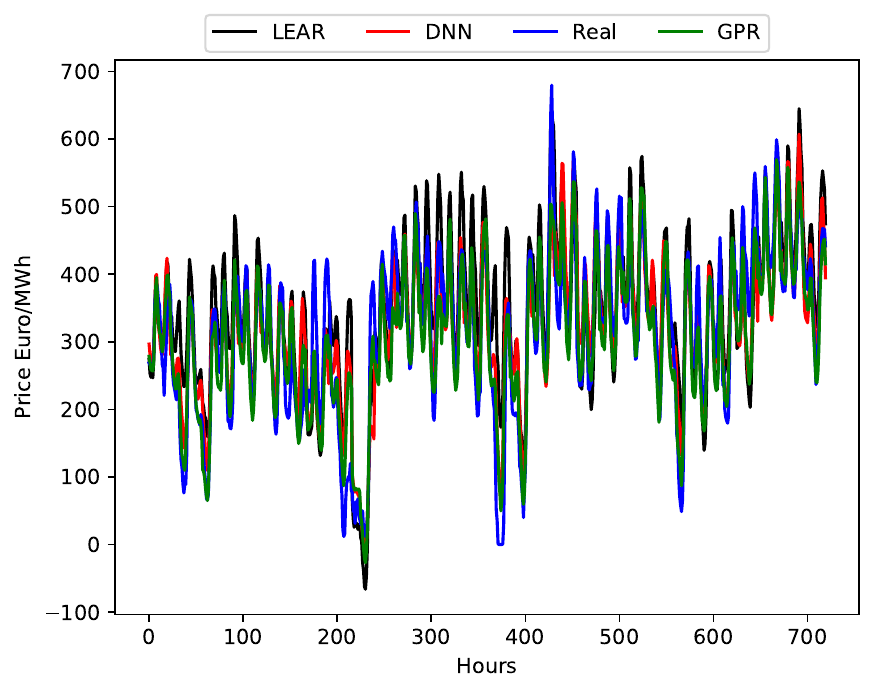}
    \caption{Gaussian Process Regression Comparison Predicted Prices for a Month of July, 2022}
    \label{gpr_com_ker_july}
\end{figure}

\begin{figure}[ht!]
    \centering
    \includegraphics[height = 5cm, width = 12cm]{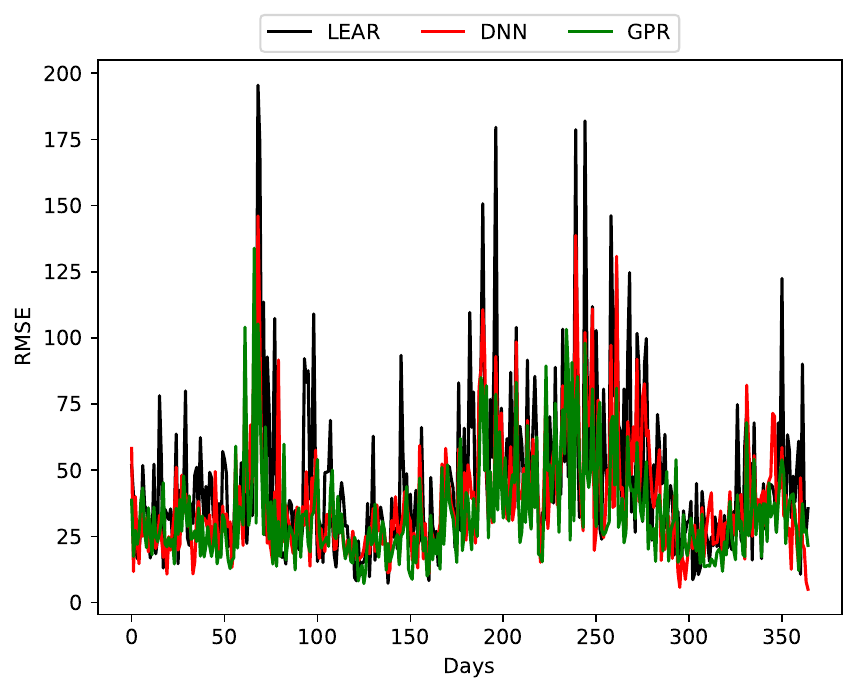}
    \caption{RMSE of Prices from January 1, 2022 to December 31, 2022}
    \label{GPR_RMSE_2022}
\end{figure}
Leveraging the enhanced functional expressiveness and covariance structure of the additive composite kernel, we adopt it as the prior covariance function in our Gaussian Process Regression (GPR) model.
The model is trained to forecast daily electricity prices for the year 2022, with predictive accuracy evaluated via the root mean square error (RMSE) metric. For qualitative assessment, Figure \ref{gpr_com_ker_july} illustrates the comparative predictions, underscoring GPR’s superior generalization capability. Quantitative results, presented in Figure \ref{GPR_RMSE_2022} and Table \ref{gpr_eval_rmse}, demonstrate that GPR consistently outperforms the benchmark models proposed described in Section \ref{s:benchmark} across the full-year horizon. While this section and the intermediate sections focus specifically on evaluating model performance for 2022, the final numerical analysis spans a broader timeframe from 2021 to 2023. The focus on 2022 for individual models evaluation is also carried to test the models with worst-case scenarios since 2022 was considered to be volatile and year deviated from typical trends in terms of energy market in Europe, primarily due to the abrupt geopolitical crisis that occurred during the year \cite{KUZEMKO2022102842}.

\begin{table}[ht!]
\centering
\caption{Errors for 2022}
\begin{tabular}{l c }
        \toprule
        Model & {RMSE}\\
        \midrule
        GPR      & 33.800\\
        LEAR      & 46.100\\
        DNN      & 38.171\\
        \bottomrule
        \end{tabular}
\label{gpr_eval_rmse}
\end{table}

However, upon visualizing the predictive uncertainty of the GPR model in Figure \ref{14_21_July_2022_com} and \ref{17_July_2022_com}, we observed that several actual price values lay outside the posterior predictive confidence intervals. This empirical deviation from the expected coverage indicates the presence of significant noise and outliers in the training data. GPR assumes that observational noise is Gaussian, however,  the noise in the electricity price data are heavy-tailed. Additionally, the price data also has abrupt spikes and shocks which contributes to the narrow bounds. Notably, this issue was not associated with a systematically higher prediction error, in fact, the point predictions remained reasonable. But the theoretical coverage of the confidence interval is shrunk due to outliers or noise. To address this, we incorporated Support Vector Regression (SVR) as an auxiliary regression model. Unlike GPR, SVR does not model predictive uncertainty explicitly but focuses on minimizing structural risk through the concept of the \(\varepsilon\)-insensitive loss function which promotes robustness by ignoring small deviations within an \(\varepsilon\)-tube around the predictions. This mechanism inherently reduces the model’s sensitivity to outliers and noise in the training data. While outlier removal is a common pre-processing step, literature warns that improper exclusion may distort the underlying data distribution and compromise generalization~\cite{8786096,9523565,9218967}. Hence, using a model inherently robust to outliers is a statistically sound alternative. To provide reliable prediction intervals alongside SVR’s point prediction, we employed conformal prediction which is a distribution-free framework that constructs prediction intervals based on empirical residuals. This yields valid coverage under minimal assumptions, producing bounds that are responsive to actual prediction errors. We discuss these in detail in next sections.

\begin{figure}[ht!]
    \centering
    \centering
        \begin{subfigure}{0.8\linewidth}
            \includegraphics[height = 4.5cm, width = 12cm]{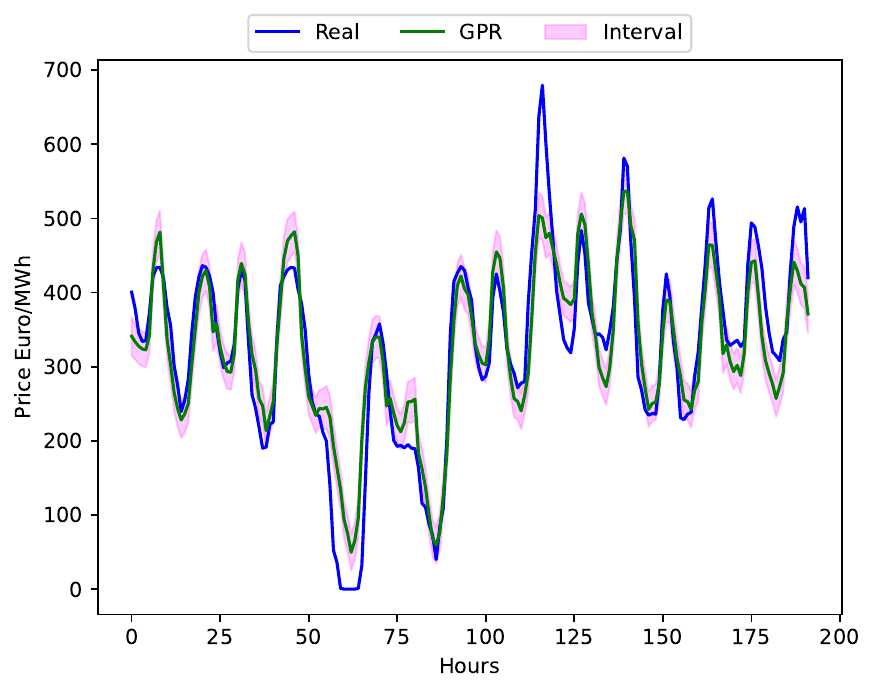}
            \caption{\centering{} }
            \label{14_21_July_2022_com}
        \end{subfigure}
        \begin{subfigure}{0.45\linewidth}
            \centering
            \includegraphics[height = 5cm, width = 6.5cm]{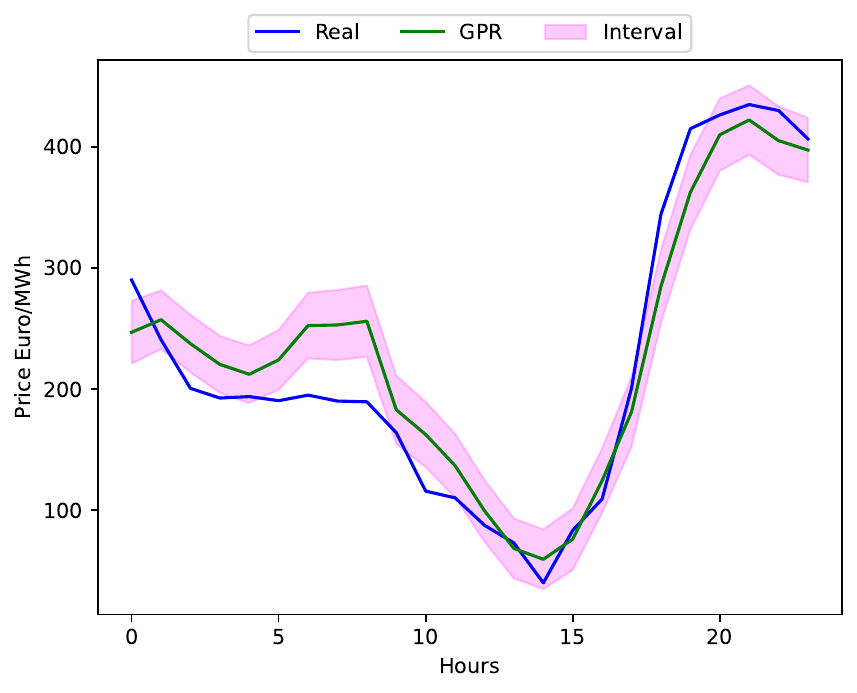}
            \caption{}
            \label{17_July_2022_com}
        \end{subfigure}
        \caption{\centering{Highlighting the predicted uncertainty via GPR: (a)For $3^{rd}$ week of July, 2022 (b) For July 17, 2022}}
        \label{interval_eval}
\end{figure}

\section{Support Vector Regression}
\label{S:SupportVectorRegression}

To overcome the problem of overfitting caused by the presence of outliers or noise in the training data, we use SVR to predict the price with the same dataset. SVR is effective when handling high-dimensional feature spaces and is resistant to overfitting, making it a suitable choice for our data \cite{ scholkopf2002learning}. Following on from the predictions via both GPR and SVR, we now combine the predictions by penalizing the predictions with higher errors, which improves the overall prediction, as discussed in Section \eqref{S:AHybridModel}.  
GPR predictions are associated with uncertainty, and hence give us information about the prediction interval based on the posterior variance. SVR prediction, on the other hand, is a statistic, and we cannot use this probabilistic approach to directly obtain the prediction interval. The solution is to apply SVR in the context of conformal prediction.

\subsection{Kernel Based Support Vector Regression}
\label{kernel based support vector regression}

Support vector machines (SVM) are supervised learning models based on margin maximization and are primarily used for classification problems \cite{suthaharan2016support}. However, a substantial literature also applies the SVM approach to regression problems, where it is referred to as support vector regression. The method is particularly useful when the relationship between the input variables and the output variables is non-linear. We provide a brief mathematical introduction below. For more information, see \cite{smola2004tutorial}.
\newline\newline
Let 
$T = \{\mathbf{t}_{1}, \cdots, \mathbf{t}_{n}\}$,
where each $t_i$ has dimension $d$, be the inputs, and 
$\{P_{1}, \cdots, P_{n}\}$ 
be the corresponding outputs. 
In our work, we have $d=3$. We aim to predict the output at 
$T^{*} = \{\mathbf{t}_{1}^{*}, \cdots, \mathbf{t}_{m}^{*}\}$. 
The basic idea of SVM is to find a function $f(\mathbf{t})$ for the input $\mathbf{t}$ that is $\epsilon$-deviated from the actual output. For simplicity, we first present the case where the input and output share a linear relationship. 
In this situation, the function can be formulated as follows:
\begin{equation}
    f(\mathbf{t}) = \left<\boldsymbol{w},\mathbf{t}\right>+b, \mathbf{t} \in T, b\in \mathbb{R},
\end{equation}
where $\boldsymbol{w}$ is a vector normal to the function and $\left<.,.\right>$ denotes the dot product. For $f(\mathbf{t})$ to be a suitable function, the difference between the real output $\{P_{i} \}_{i=1}^{n}$ and the value of the function $f(\mathbf{t}_{i}),i=1,\cdots,n$ can be up to $\epsilon$, and the $\boldsymbol{w}$ should be flat, which in the linear case means that it should be sufficiently small. 
This can be formulated as the following optimization problem:
\begin{equation}\label{svm_convex}
\begin{split}
    &\text{$\min$ $\frac{1}{2}\vert\vert \boldsymbol{w} \vert\vert^{2}$}\\ 
    & \text{subject to} \begin{cases}
        P_{i} \text{\textminus} \left<\boldsymbol{w},\mathbf{t}\right> \leq \epsilon \\
        \left<\boldsymbol{w},\mathbf{t}\right> \text{\textminus} P_{i} \leq \epsilon 
    \end{cases}
    .
\end{split}
\end{equation}
Problem (\ref{svm_convex}) is convex and we assume that it is feasible. Practically, the case of infeasibility may arise. In order to deal with this condition, we can introduce the slack variables $\xi$ and $\xi^{*}$, as in \cite{cortes1995support}, and the optimization problem (\ref{svm_convex}) can be reformulated as follows:
\begin{equation}\label{svm_convex_slack}
\begin{split}
    &\text{$\min_{\boldsymbol{w,b,\xi,\xi^{*}}}$ $\frac{1}{2}\vert\vert \boldsymbol{w} \vert\vert^{2} + C\Sigma_{i=1}^{n}(\xi_{i} + \xi_{i}^{*})$}\\ 
    & \text{subject to} \begin{cases}
        P_{i} \text{\textminus} \left<\boldsymbol{w},\mathbf{t}\right> \leq \epsilon + \xi_{i}\\
        \left<\boldsymbol{w},\mathbf{t}\right> \text{\textminus} P_{i} \leq \epsilon + \xi_{i}^{*}\\
        \xi_{i}, \xi_{i}^{*}\geq 0
    \end{cases}
\end{split}
\end{equation}
By using the Lagrangian method and constructing the dual of the problem (\ref{svm_convex_slack}), the problem can be solved as in \cite{mangasarian19651969}. If the input and output have a non-linear relationship, the function can be defined as follows:
\begin{equation}\label{kernel_trick}
    f(\mathbf{t}) = \sum_{i=1}^{n}(\alpha_{i} \text{\textminus} \alpha_{i}^{*})K_{SVR}(\mathbf{t}_{i}, \mathbf{t})+b,
\end{equation}
where $\alpha_{i}$ and $\alpha_{i}^{*}$ are the Lagrange multipliers for Problem (\ref{svm_convex_slack}) and $K_{SVR}(\mathbf{t}_{i},\mathbf{t})$ is a kernel. The formulation of the function in Equation (\ref{kernel_trick}) is often referred to as the \textit{kernel trick}.
When formulating the Lagrangian form of eq. \eqref{svm_convex_slack}, two more Lagrange multipliers appear, but they can be eliminated as in \cite[Section 1.3]{smola2004tutorial}. 
The kernel trick method is useful because the kernel $K_{SVR}$ then maps the $d-$dimensional input to a higher dimensional space in which the mapped input is approximately linear to the output, which reduces the complexity. 
The mapping of the input to higher dimensional data is implicit as follows:
\begin{equation}\label{dot_prod}
    K_{SVR}(m,n) = \left<\phi(m),\phi(n)\right>
\end{equation}
where $\phi:T\rightarrow \mathbb{R}^{l}$, and $l>>d$. It is important to note that the kernel must be in the form of a dot product of functions that map the input to higher dimensional space, as shown in eq. \eqref{dot_prod} \cite[Section 2.3, Theorem 2]{smola2004tutorial}. 
Since the kernel only depends on the dot product of the inputs, the explicit form of $\phi$ is not required.

\begin{figure}[ht!]
    \centering
    \includegraphics[height = 5cm, width = 12cm]{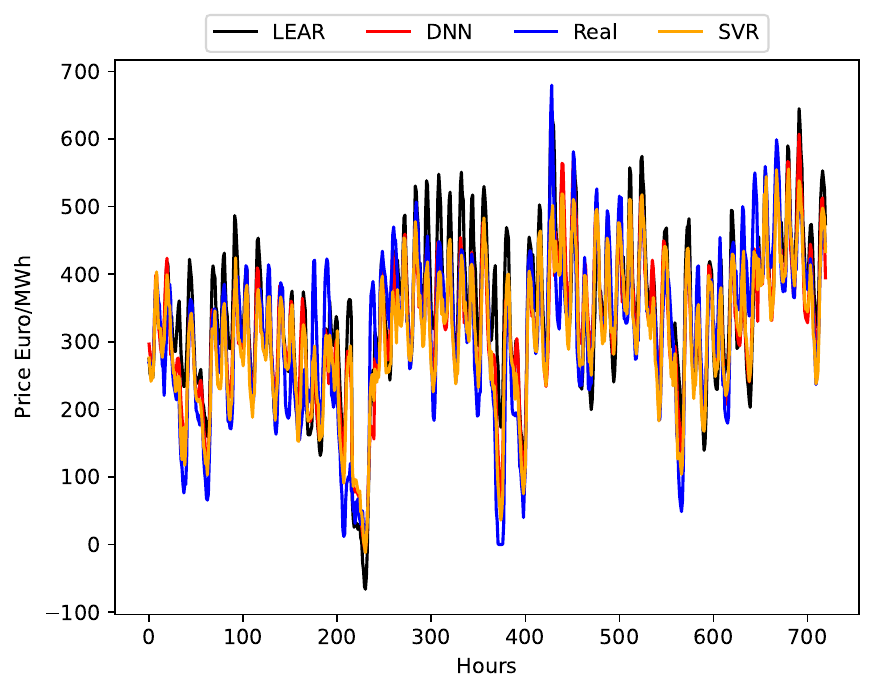}
    \caption{Support Vector Regression Comparison Predicted Prices for a Month of July, 2022}
    \label{svr_com_ker_july}
\end{figure}

\begin{figure}[ht!]
    \centering
    \includegraphics[height = 5cm, width = 12cm]{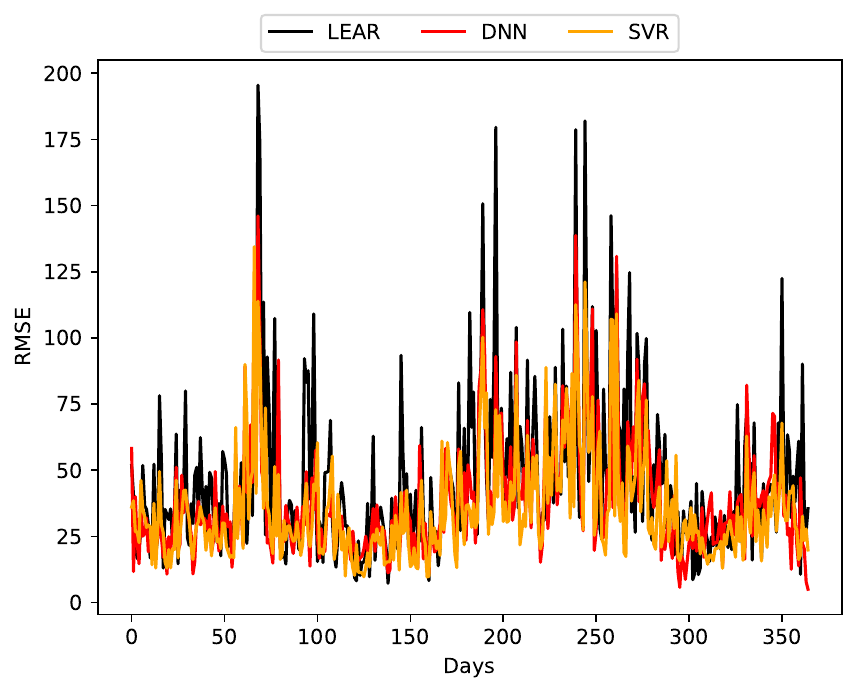}
    \caption{RMSE of Prices Predicted Via Support Vector Regression from January 1, 2022 to December 31, 2022}
    \label{SVR_RMSE_2022}
\end{figure}

Figure~\ref{svr_com_ker_july} presents the SVR-based electricity price predictions for July 2022, where the model parameters and kernel type were selected via grid search to ensure optimal performance. To evaluate the model's performance over the entire year, we compared the SVR predictions against the benchmark models across 2022. As shown in Figure~\ref{SVR_RMSE_2022} and Table~\ref{svr_eval_rmse}, SVR demonstrates improved performance relative to the benchmarks. Notably, the SVR model was trained using the same input-output structure as the GPR model to ensure consistency. However, Table~\ref{svr_eval_rmse} also indicates that GPR achieves higher point-prediction accuracy in terms of RMSE. Since our goal is not to compare SVR and GPR in isolation but to combine their strengths, the RMSE results support the use of a linear combination of their predictions for improved overall performance.

\begin{table}[ht!]
\centering
\caption{Errors for 2022}
\begin{tabular}{l c }
        \toprule
        Model & {RMSE}\\
        \midrule
        GPR      & 33.800\\
        LEAR      & 46.100\\
        DNN      & 38.171\\
        SVR & 35.496\\
        \bottomrule
        \end{tabular}
\label{svr_eval_rmse}
\end{table}


\subsection{Prediction Intervals for Support Vector Regression}
\label{SS:PredictionIntervals_SVR}

To compute the conformal predictions, we follow \cite{10.3150/21-BEJ1447}. 
Let $\{P_{1}, \cdots, P_{n}\}$ be the true values (prices). The corresponding predicted values using SVR are $\{\hat{P}_{1}, \cdots, \hat{P}_{n}\}$ for the inputs
$T=\{\mathbf{t}_{1}, \cdots, \mathbf{t}_{n}\}$. 
Similarly to GPR, the output at a new point $\mathbf{t}^{*}\in T^{*}$ is denoted by $P^{*}$. We want to construct the prediction intervals for these so-called non-conformity values. 
In simple terms, non-conformity values tell us how different a test point is compared with a set of training points \cite{10.3150/21-BEJ1447}. 
This difference is denoted by $\alpha$ and is calculated as follows:
\begin{equation}
\label{non-conformity}
    \alpha_{i} = |P_{i} \text{\textminus} \hat{P}_{i}|.
\end{equation}
Since we do not have access to the real value for $P^{*}$, on the basis of the training output, we assume that $\Tilde{P}_{j}$ belongs to the interval: 
\begin{equation*}
\Bigl[ 
P_{low} = \mu_{svr} \text{\textminus} \nu \sigma_{svr}, \;
P_{up}  = \mu_{svr} + \nu \sigma_{svr} 
\Bigr], 
\end{equation*}
where $\mu_{svr}$ and $\sigma_{svr}$ are the mean and standard deviation of the training outputs. 
Then: 
\begin{equation}
\label{non-conformity_new_alpha}
    \alpha^{*}_{j} = |P^{*} \text{\textminus} \Tilde{P}_{j}| 
    \text{ for } \Tilde{P}_{j} \in [P_{low}, P_{up}].
\end{equation}
We have chosen $j=1, \cdots, 500$ for $\Tilde{P}_{j}$, which we assume to be uniformly distributed between $P_{low} \text{ and } P_{up}$. 
For the $95$\% confidence interval, we compute the proportionality values, denoted by $\Gamma_{j}$, as follows:
\begin{equation}\label{p_value}
    \Gamma_{j} =\frac{\text{n}\{\alpha_{i}:\alpha_{i}\geq\alpha^{*}_{j}, i=1,\cdots,n\}}{n+1}, \text{  } \forall j = 1,\cdots, 500,
\end{equation}
where $\text{n}\{\cdot\}$ denotes a set's cardinality and $n$ is the number of training outputs.

Next, we construct a set $\Pi$ of values as follows:
\begin{equation*}
    \Pi = \left\{\Tilde{P}_{j}:\frac{\text{n}\{\alpha_{i}:\alpha_{i}\geq|P^{*} \text{\textminus} \Tilde{P}_{j}|, i=1,\cdots,n\}}{n+1}\geq 0.05, \forall j = 1,\cdots, 500\right\}.
\end{equation*} 
The prediction interval $I_{SVR}$ for $P^{*}$ is then given by:
\begin{eqnarray*}
    I_{SVR} & = & [lb_{SVR},ub_{SVR}] \hspace{0.5cm} \text{where} \\
    lb_{SVR} & = & \min(\Pi)\text{ and } ub_{SVR}=\max(\Pi)
\end{eqnarray*}
Note, we assume that the data are exchangeable, which is a weaker assumption than independent and identically distributed data.
\hfill\newline
Since the $\Tilde{P}_{j}$'s are drawn uniformly and the interval $I_{SVR}$ depends on them, we use bootstrapping and repeatedly generate $\Tilde{P}_{j}$'s to achieve the robustness of the interval. 
For each repetition, we obtain an interval
$I_{SVR}^{k} = [lb_{SVR}^{(k)}, ub_{SVR}^{(k)}], k=1,\cdots,s$.  
We then compute the average of the lower and upper bounds of these $I_{SVR}^{k}$'s as a final prediction interval, which is given by:
\begin{equation}\label{svr_interval_conformal}
    I_{SVR}^{BS} = [lb_{SVR}^{BS},ub_{SVR}^{BS}] \text{ where }
\end{equation}
$$lb_{SVR}^{BS} = \frac{1}{s}\sum_{k = 1}^{s}lb_{SVR}^{(k)} \text{ and } ub_{SVR}^{BS} = \frac{1}{s}\sum_{k = 1}^{s}ub_{SVR}^{(k)}$$
\begin{figure}[ht!]
    \centering
    \includegraphics[height = 5cm, width = 12cm]{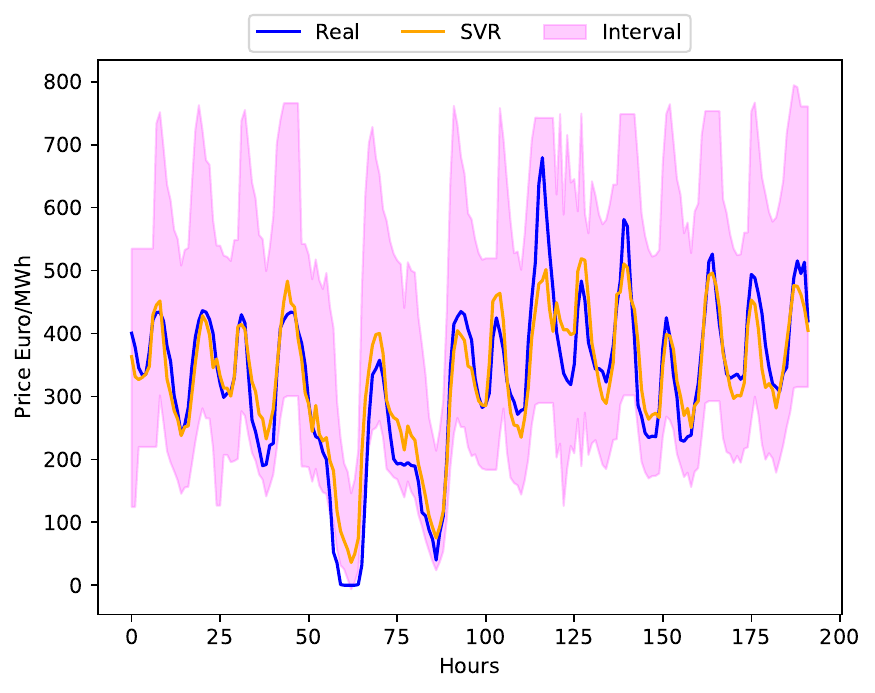}
    \caption{\centering{Highlighting the predicted uncertainty via SVR for $3^{rd}$ week of July, 2022} }
    \label{14_21_July_2022_svr}
\end{figure}
In Figure \ref{14_21_July_2022_svr} and \ref{17_July_2022_svr} we show estimated the lower and upper bounds for each price predicted via SVR for the $3^{rd}$ week of July, 2022, using the aforementioned approached from equation (\ref{non-conformity}) to (\ref{svr_interval_conformal}). On comparing this with the via GPR in Figure \ref{interval_eval} we can can conclude that the issue in the GPR case is eliminated. However, the bounds are wider indicates the higher uncertainty. 
\begin{figure}[ht!]
    \centering
    \includegraphics[height = 5cm, width = 6.5cm]{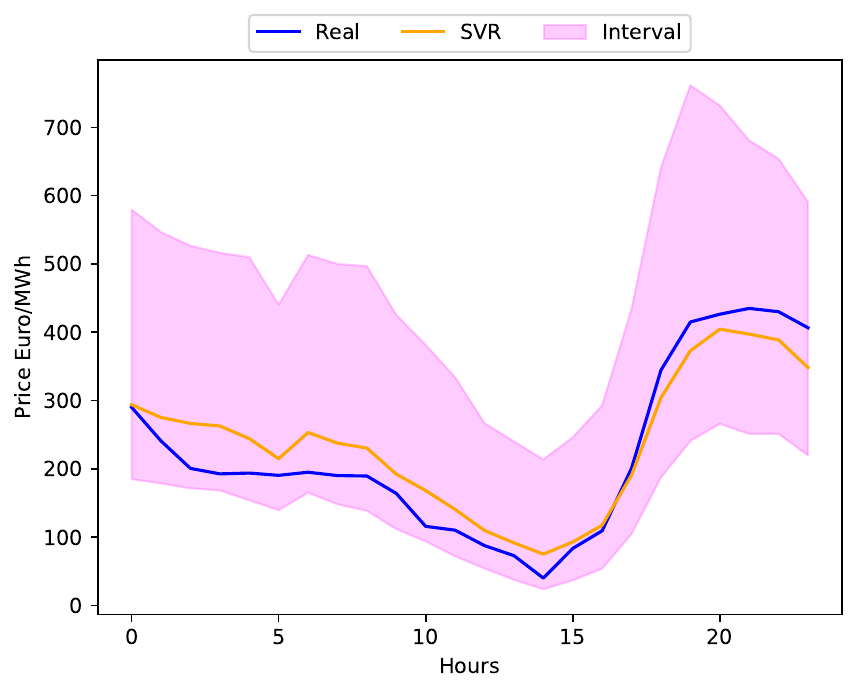}
    \caption{\centering{Highlighting the predicted uncertainty via SVR for July 17, 2022} }
     \label{17_July_2022_svr}
\end{figure}

At this point, it is clear that both GPR and SVR contribute to the prediction task, each with distinct strengths and limitations. GPR yields more accurate point estimates but produces less reliable uncertainty bounds, while SVR, though less accurate in point prediction, offers more robust interval estimates. This complementarity motivates the combination of both models to mitigate individual weaknesses and enhance predictive performance in both point estimation and interval coverage. 


\section{A Hybrid Model for Electricity Price Prediction}
\label{S:AHybridModel}

\subsection{The Hybrid Model}
\label{SS:TheHybridModel}

In this section, we propose a hybrid model that linearly combines GPR-based and SVR-based prediction. Since SVR is better at dealing with outliers \cite{chuang2011hybrid}, 
it helps to assign lesser weights to the GPR prediction when the predictions are affected by outliers in the training data. As mentioned in Section \ref{S:GaussianProcessRegression}, we have chosen a 248-dimensional input and we have the following training data:
\begin{equation}
\label{train_set}
\mathbf{P}_{\text{train}} 
= \{P_{\mathbf{t}_{i}},: \mathbf{t}_{i} \in T \subset \mathbb{R}^{248}, \forall i=1,\cdots n\}.
\end{equation}
where $$\mathbf{t}_{i} = [i, P^{(i-1)}, P^{(i-2)}, P^{(i-3)}, P^{(i-7)}, L^{(i)}, L^{(i-1)}, L^{(i-7)}, R^{(i)}, R^{(i-1)}, R^{(i-7)}, BD^{(i)}].$$
To predict the price, we define a hybrid model as the linear combination of individual predictions, which is given as follows:
\begin{equation}\label{final_model}
    \mathbf{P_{x^{*}}} = \lambda_{1}\underbrace{\left(\Sigma_{\mathbf{P}^{*},\mathbf{P}} \left(\Sigma_{\mathbf{P}} + \sigma_{n}\mathbf{I}\right)^{\text{\textminus}1}\mathbf{P}\right)}_{\text{Gaussian process regression}} + \lambda_{2}\overbrace{\left(\sum_{i=1}^{n}(\alpha_{i} \text{\textminus} \alpha_{i}^{*})K_{SVR}(\mathbf{x}_{i}, \mathbf{x})+b\right)}^{\text{Support vector regression}},
\end{equation}
where the first expression on the right-hand side denotes the GPR prediction and the second the SVR prediction. The parameters of the hybrid model are estimated individually. In this study we have chosen $\lambda_{1} = \lambda_{2} = 0.5$. The two prediction intervals of the individual models are combined using the weights $\lambda_{1}$ and $\lambda_{2}$ as follows:
\begin{equation}
    I_{GPR+SVR} = [lb_{com},ub_{com}],
\end{equation}
with:
\begin{align*}
lb_{com} &= \lambda_{1} \, lb_{GPR} + \lambda_{2} \, lb_{SVR}^{BS}, \\
ub_{com} &= \lambda_{1} \, ub_{GPR} + \lambda_{2} \, ub_{SVR}^{BS}.
\end{align*}
For the GPR expression in eq. \eqref{final_model}, the covariance matrix $\Sigma_{\mathbf{P}}$ is given by $K_{GPR}$, which is defined as the sum of the squared exponential covariance function and the rational quadratic function as follows:
\begin{equation}
\label{se+ratquad}
    \begin{split}
        K_{\text{GPR}}(\mathbf{t}_{i},\mathbf{t}_{j}) 
        &=  K_{se}(\mathbf{t}_{i},\mathbf{t}_{j}) + K_{rq}(\mathbf{t}_{i},\mathbf{t}_{j}) \\
        &= \sigma^{2}_{se}\exp\left(\text{\textminus}\frac{\vert\vert \mathbf{t}_{i} \text{\textminus} \mathbf{t}_{j} \vert\vert^{2}}{2\ell^{2}_{se}}\right) +  \sigma^{2}_{rq}\left(1+\frac{\vert\vert\mathbf{t}_{i} \text{\textminus} \mathbf{t}_{j}\vert\vert^{2}}{2\alpha \ell_{rq}}\right)^{\text{\textminus}\alpha}, 
    \end{split}
\end{equation}
where $\ell_{se}$ and $\ell_{rq}$ are the length scales for the squared exponential and the rational quadratic kernel, respectively, and $\sigma_{se}$ and $\sigma_{rq}$ are the variance parameters for the squared exponential and the rational quadratic kernel, respectively. These are the hyperparameters of Gaussian process models and are estimated using $\mathbf{P}_{\text{train}}$ via MLE. 
The marginal likelihood function $\mathbf{P_{\text{train}}}$ is given as follows:
\begin{equation}\label{MLPtrain}
    p(\mathbf{P}_{\text{train}} | T, \theta) = \mathcal{N}(\mathbf{P}_{\text{train}} | \mathbf{0}, K_{\text{GPR}} + \sigma_n^2 \mathbf{I}).
\end{equation}
Here, 
$\theta = (\sigma^2_{\text{se}}, \ell_{\text{se}}, \sigma^2_{\text{raq}}, \ell_{\text{raq}}, \alpha, \sigma_n^2)$, 
$\mathbf{I}$ is the identity matrix, and $\sigma^{2}_{n}$ is the variance of the noise term. It is important to include the noise term in the prior because we assume that the data are noisy.
The log marginal likelihood (LML) is given by:
\begin{equation}\label{LMLPtrain}
\begin{split}
    \log \left[p(\mathbf{P}_{\text{train}} | T, \theta)\right] = & \text{\textminus} \frac{1}{2} \mathbf{P}_{\text{train}}^{\mathrm{T}} (K_{\text{GPR}} + \sigma_n^2 \mathbf{I})^{\text{\textminus}1} \mathbf{P}_{\text{train}}\\
    & \text{\textminus} \frac{1}{2} \log \left| K_{\text{GPR}} + \sigma_n^2 \mathbf{I} \right| \text{\textminus} \frac{n}{2} \log 2\pi.
\end{split}
\end{equation}
By differentiating the log marginal likelihood with respect to the parameters individually and setting them to zero, we obtain parameter estimates.

\hfill\newline\newline 
For the SVR, we used the same input to train the model, but changed the kernel. 
GPR and SVR are both kernel-based approaches, but they use the kernel in different ways. 
In GPR, the kernel defines the smoothness and variability of the predicting distribution. 
In SVR, the kernel is used to transform the input space to a higher-dimensional space, in which the input and output share an approximately linear relationship. 
We chose the best-performing kernel from a set of kernels that comprises the squared exponential, polynomial, sigmoid, and linear kernels. 
We used a grid search to choose the kernel parameters for each prediction. 
To compute $b$ in the final model given in eq. \eqref{final_model}, we applied the Karush-Kuhn-Tucker (KKT) condition \cite{karush1939minima}. 
For a detailed derivation of $b$, please refer to \cite{lin2005library}.


\section{Numerical Results}
\label{S:NumericalResults}

\subsection{Predictions made with the Hybrid Model}
\label{SS:PredictionsHybri Model}

We tested our new hybrid model on German power prices for the years 2021 to 2023. We focused on each year individually considering each year having different market trends \cite{PierreSchneider2024}. The seasonal comparison (summer, winter, autumn, and spring) was only conducted for all three years, however, we only present the plots for 2023 in this section and for rest of the years we have put the plots in the Appendix \ref{A:SeasonalComparisons}. The GPR and SVR models were trained using the past $365$ days of data. The training data were constructed as outlined in eq. \eqref{train_set}. We denote the training input and training output as $\text{train}_{in}$ and $\text{train}_{out}$, respectively:
\begin{eqnarray*}
\text{train}_{in} = \left( \mathbf{t}_{1}, \cdots, \mathbf{t}_{365} \right)^{\top} 
\text{ and } 
\text{train}_{out} = \left( P^{(1)}, \cdots, P^{(365)} \right)^{\top}.
\end{eqnarray*}
The hyperparameters of the Gaussian process are estimated using MLE, as shown in eq. \eqref{LMLPtrain}, based on $\text{train}_{in}$ and $\text{train}_{out}$. 
Since the training data are transformed, the model parameters are learned on this transformed scale, and predictions are subsequently re-transformed back to the original scale.
For SVR the same training data are used to choose the parameters for the squared exponential kernel using a grid search. For the search process, we chose $\epsilon \in [0.001,0.1,0.1]$, $C\in [0.1,1,10]$, $\hat{c} \in [0,1,10]$. The kernel for SVR training is also chosen via grid search where the potential kernels are \{squared exponential, polynomial, linear, sigmoid\}. Both trained GPR and SVR models are used to predict the next $48$ data points, i.e., two days, using the same prediction inputs denoted here as $\text{pred}_{in}$.
\begin{figure}[ht!]
    \centering
    \centering
        \begin{subfigure}{0.8\linewidth}
            \includegraphics[height = 5cm, width = 12cm]{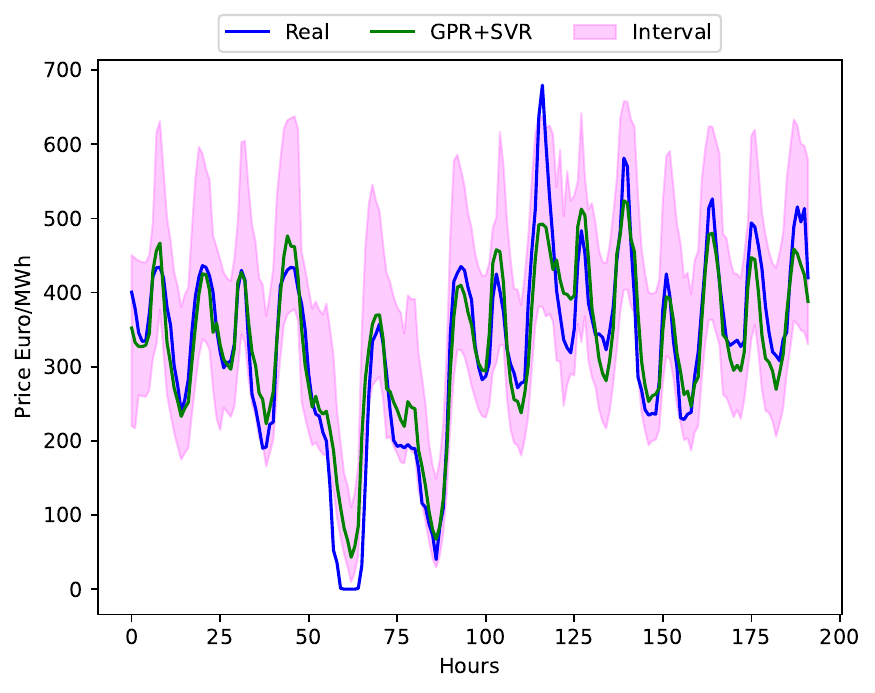}
            \caption{\centering{} }
            \label{14_21_July_2022_gpr_svr}
        \end{subfigure}
        \begin{subfigure}{0.45\linewidth}
            \centering
            \includegraphics[height = 5cm, width = 6.5cm]{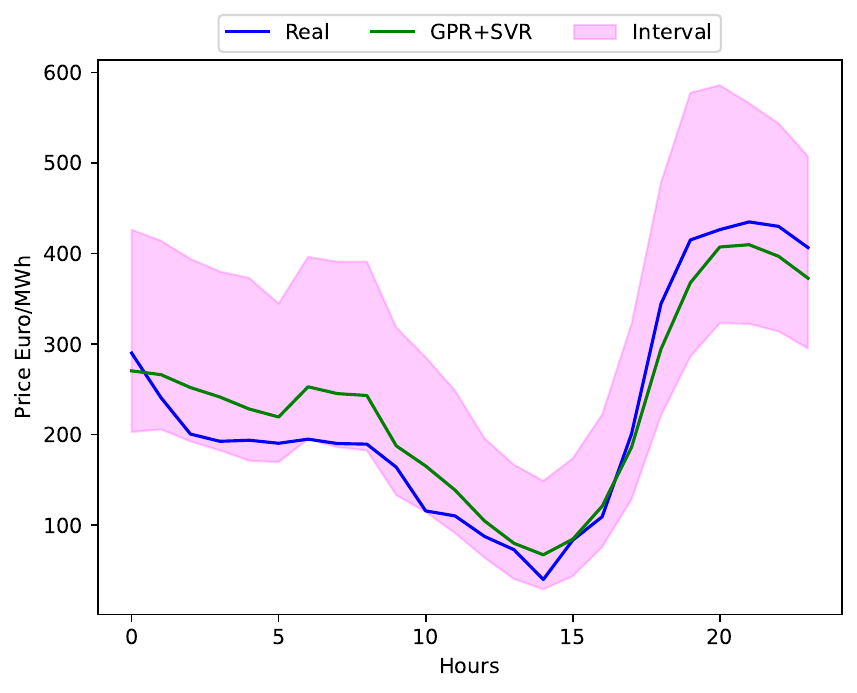}
            \caption{}
            \label{17_July_2022_gpr_svr}
        \end{subfigure}
        \caption{\centering{Highlighting the predicted uncertainty via GPR: (a)For $3^{rd}$ week of July, 2022 (b) For July 17, 2022}}
        \label{interval_eval_gpr_svr}
\end{figure}

For example, if we are predicting the price for July 17, 2022, we take the hourly data from July 17, 2021 to July 16, 2022 as a vector of $365$ data points and predict the prices for July 17, 2022. Each data points in the vector is of dimension 248 and are designed as described in \ref{SS:TheHybridModel}.
\begin{eqnarray*}
\text{pred}_{in} = \left[ \mathbf{t}_{366} \right] 
\text{ and } 
\text{pred}_{out} = \left[ P^{(366)} \right]
\end{eqnarray*}
For each 24 hour we calibrate the model and and predict the individually which gives 24 predictions for a day. The same process is repeated for the next day and this continues till the prediction horizon.

For a visual comparison, we show the prediction for 8 days and for one-day prices (consecutive days chosen at random) in Figures \ref{14_21_July_2022_gpr_svr} and \ref{17_July_2022_gpr_svr}, respectively. These Figures compare the real price with the predictions from the hybrid model.

\begin{figure}[ht!]
        \centering
        \begin{subfigure}{0.5\linewidth}
            \includegraphics[height = 5cm, width = 7cm]{ 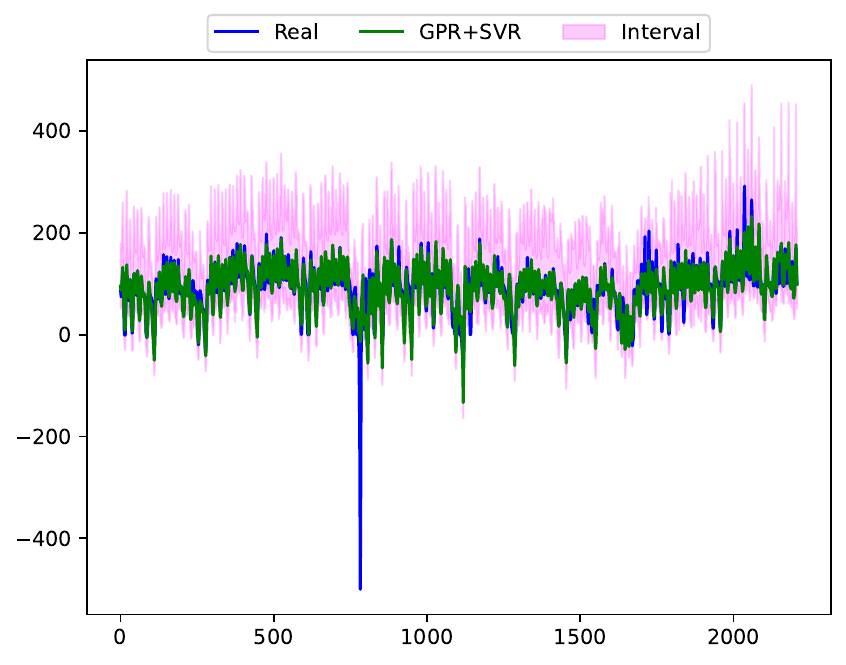}
            \caption{}
            \label{summer}
        \end{subfigure}
        \begin{subfigure}{0.45\linewidth}
            \includegraphics[height = 5cm, width = 7cm]{ 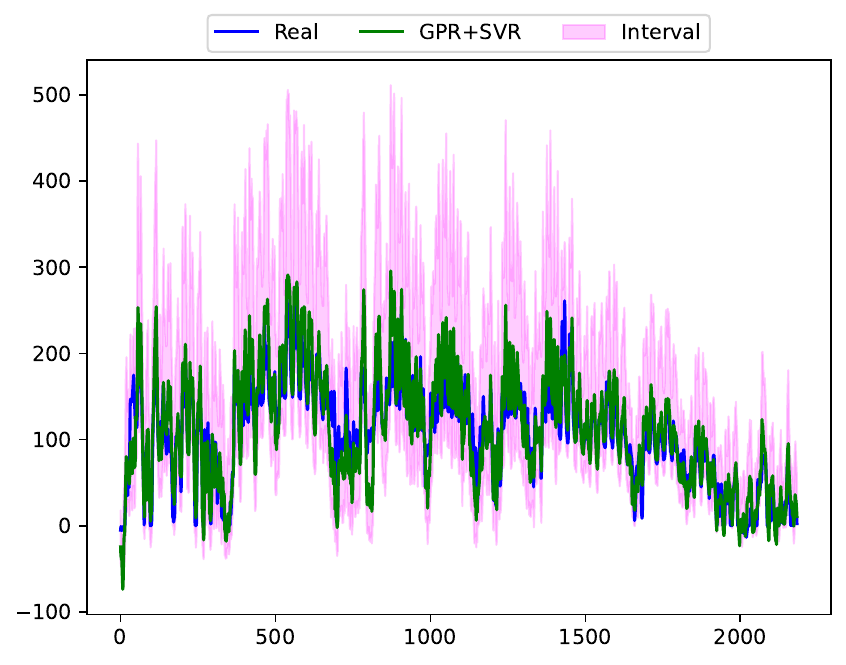}
            \caption{}
            \label{winter}
        \end{subfigure}
        \begin{subfigure}{0.5\linewidth}
            \includegraphics[height = 5cm, width = 7cm]{ 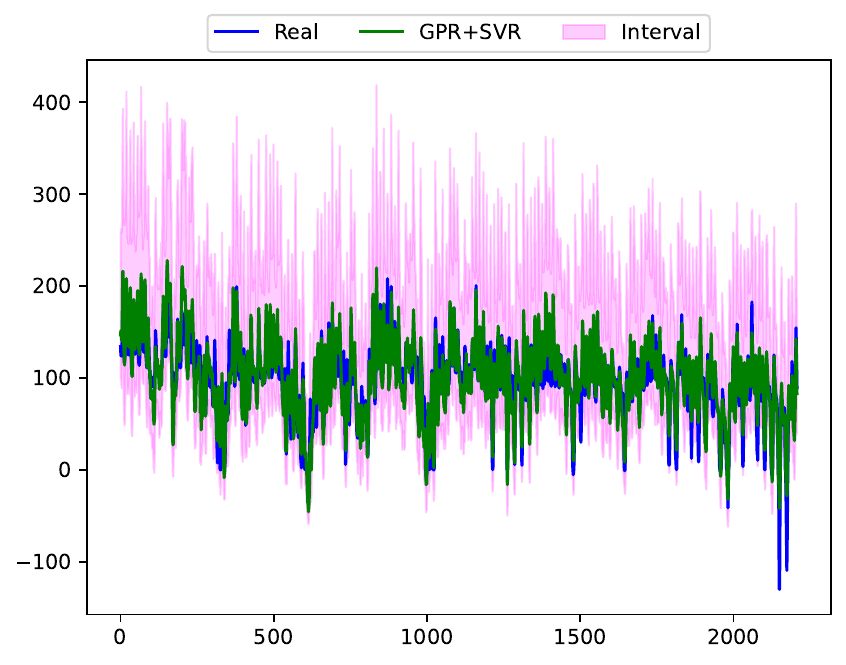}
            \caption{}
            \label{spring}
        \end{subfigure}
        \begin{subfigure}{0.45\linewidth}
            \includegraphics[height = 5cm, width = 7cm]{ 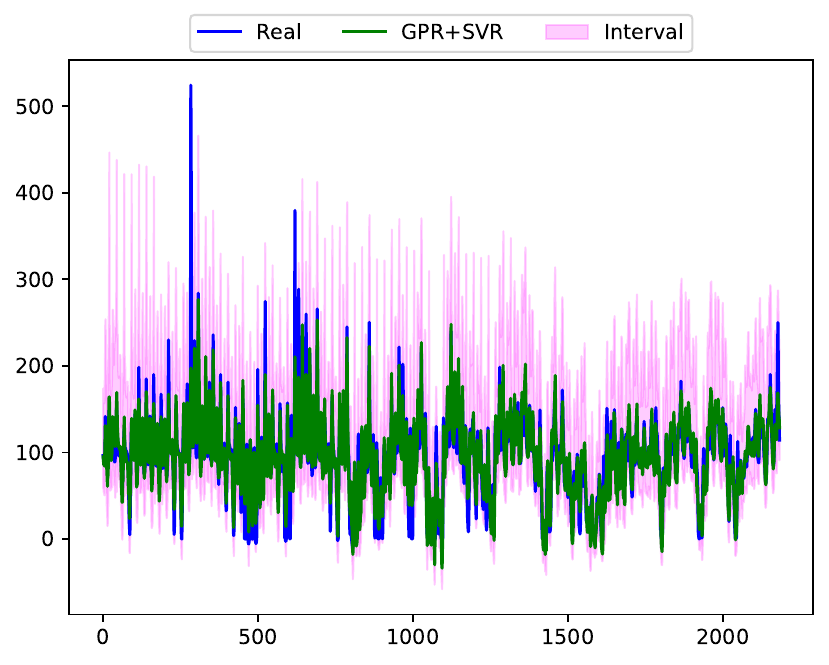}
            \caption{}
            \label{autumn}
        \end{subfigure}
        \caption{\centering{Seasonal Comparison: (a) Summer of 2023}, (b) Winter of 2023, (c) Spring of 2023 and (d) Autumn of 2023}
        \label{seasonal prediction}
\end{figure}

However, for the year 2023, model performance is evaluated using $365$ days, i.e., from January 1, 2023 to December 31, 2023. In Figure \ref{seasonal prediction}, we show the prediction based on the hybrid model for four seasons in Germany, namely summer (June to August), winter(January, February, and December of the same year), autumn (September to November), and spring (March to May) for year 2023. From Figures \ref{summer}, \ref{winter}, \ref{spring}, \ref{autumn}, we see that, on average, the predictions are aligned with the real data. For a few days in each season, the prediction for some time points is comparatively worse than the price predictions in neighboring time points, however, such cases are not repetitive. This is particularly evident when the training data suddenly have very high or very low prices for very short intervals. In such cases, the model cannot capture the pattern within the short interval and the prediction is relatively poor. 


\subsection{Error Analysis}
\label{SS:ErrorAnalysis}

The performance of our hybrid model is evaluated by comparing its point predictions with the benchmark models in Section \ref{s:benchmark}. The evaluation was carried out over three distinct forecasting years: 2021, 2022, and 2023. The root mean square error (RMSE), mean absolute error (MAE), mean absolute percentage error and symmetric mean absolute error (sMAPE) were used to evaluate the model performance and are defined here as follows:
\begin{align}
    \text{RMSE}_{i} = \sqrt{ \frac{1}{m} \sum_{j=1}^{m} \left(P_{j} \text{\textminus} \hat{P_{j}} \right)^2}\hspace{0.3cm} 
    &\text{ and } \hspace{0.3cm} 
    \text{Error\_Score(RMSE)} = \frac{1}{N}\sum_{i=1}^{N} \text{RMSE}_{i}, \label{e:rmse}\\
    \text{MAE}_{i} = \frac{1}{m}\sum_{j=1}^{m}|P_{j} \text{\textminus} \hat{P}_{j}|\hspace{0.3cm} 
    &\text{ and } \hspace{0.3cm} 
    \text{Error\_Score(MAE)} = \frac{1}{N}\sum_{i=1}^{N} \text{MAE}_{i}\label{e:mae}.\\
    \text{MAPE}_{i} = \sqrt{ \frac{1}{m} \sum_{j=1}^{m} \left\vert \frac{P_{j} \text{\textminus} \hat{P_{j}}}{P_{j}} \right\vert}\hspace{0.3cm} 
    &\text{ and } \hspace{0.3cm} 
    \text{Error\_Score(MAPE)} = \frac{1}{N}\sum_{i=1}^{N} \text{MAPE}_{i}, \label{e:mape}\\
    \text{sMAPE}_{i} = \sqrt{ \frac{1}{m} \sum_{j=1}^{m}  2*\frac{\left\vert P_{j} \text{\textminus} \hat{P_{j}}\right\vert}{\vert P_{j}\vert + \vert \hat{P}_{j}\vert }  }\hspace{0.3cm} 
    &\text{ and } \hspace{0.3cm} 
    \text{Error\_Score(sMAPE)} = \frac{1}{N}\sum_{i=1}^{N} \text{sMAPE}_{i}, \label{e:smape}
\end{align}

Here, $P_{j}$ and $\hat{P_{j}}$ are the real and predicted values of hour $j$ of day $i$, while $m$ and $N$ are the number of days and the number of hours in a day.
While descriptive metrics offer insight into model performance, they do not provide information on the statistical significance of the differences in predictive accuracy. Hence, to supplement the error analysis with inferential statistical evidence, we employ formal hypothesis testing. Specifically: 
\begin{itemize}
    \item \textbf{Diebold–Mariano (DM) Test:} To assess the pairwise predictive accuracy between models, we utilize the Diebold–Mariano test. This test evaluates the null hypothesis that two forecasting models have equal expected loss, based on the difference in prediction errors over time. It is applied individually for each of the three forecast years to compare the hybrid model (GPR+SVR) against each baseline model.
    
    \item \textbf{Friedman Test with Nemenyi Post-hoc Analysis:} To evaluate and rank multiple forecasting models simultaneously across different datasets or time periods, we further employ the Friedman test — a non-parametric statistical test for detecting differences in the median ranks of several related models. In cases where the null hypothesis of the Friedman test is rejected, we apply the Nemenyi post-hoc test to perform pairwise comparisons between models, identifying which differences are statistically significant. These analyses are reported in Section~\ref{subsec:fr_test}.
\end{itemize}

The inclusion of these statistical tests ensures that claims of superior model performance are not merely based on numerical metrics but are backed by statistically significant evidence. This dual-layered evaluation through both error metrics and hypothesis testing enhances the robustness and credibility of the proposed methodology.

The evaluation of intervals were done using the Prediction Interval Coverage Probability (PICP) and Mean Prediction Interval Width (MPIW). PICP measures the proportion of true target values that fall within the corresponding predictive intervals. Formally, for $n$ predictions, the PICP is defined as:

\begin{equation}
\text{PICP} = \frac{1}{24} \sum_{h=1}^{24} \mathbb{I}\left( P^{(i)}_{h} \in [\hat{L}, \hat{U}] \right),
\end{equation}

where $P^{(i)}_{h}$ is the true price of $h^{th}$ hour of day $i$, $\hat{L}$ and $\hat{U}$ denote the lower and upper bounds of the estimated prediction interval, and $\mathbb{I}(\cdot)$ is the indicator function that equals 1 if the condition is true and 0 otherwise. A well-calibrated model with nominal coverage level $\alpha$ (e.g., $95\%$) should achieve $\text{PICP} \approx \alpha$.

where as MPIW captures the average width of the prediction intervals and reflects their sharpness. It is computed as:

\begin{equation}
\text{MPIW} = \frac{1}{24} \sum_{h=1}^{24} \left( \hat{U} -\hat{L} \right).
\end{equation}

Lower values of MPIW indicate tighter intervals, which are desirable provided that the coverage (PICP) is not compromised. Together, PICP and MPIW offer a balanced evaluation of the model's uncertainty estimates in terms of both calibration and informativeness.
  
In our study, for each year 2021 to 2023, we have $m=24$ and $N=365$, which gives the error scores for daily prices. Across all metrics and forecast years, the proposed GPR+SVR model consistently ranks among the top-performing models. For the visual comparison the error plots using RMSE is only shown, however, the evaluation is carried out using all of the above mentioned error metrics. As shown in Table \ref{mae_error}, it achieves the lowest MAE in all three years, outperforming even its constituent models (GPR and SVR), as well as the benchmarks: DNN and LEAR. The improvement over LEAR and DNN is particularly noticeable in 2022, where the GPR+SVR model reduces MAE by approximately 29\% and 12.2\%, respectively. The numerical values are also supported by the Figure \ref{RMSE_2023}, \ref{RMSE_2022} and \ref{RMSE_2021}.

\begin{figure}[ht!]
    \centering
    \includegraphics[height = 4.5cm, width = 11cm]{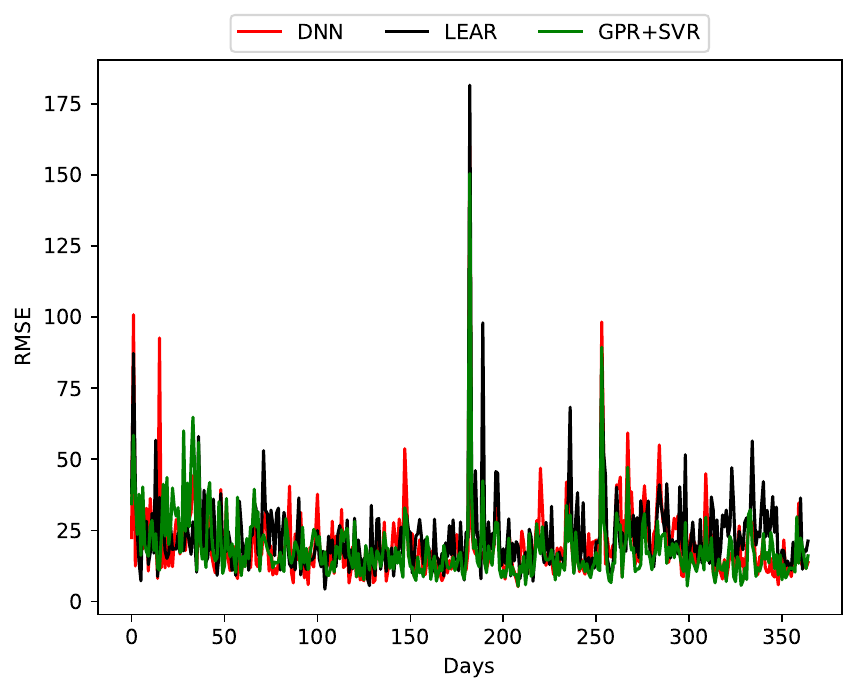}
        \caption{RMSE for 2023}
        \label{RMSE_2023}
\end{figure}
\begin{figure}[ht!]
    \centering
    \includegraphics[height = 4.5cm, width = 11cm]{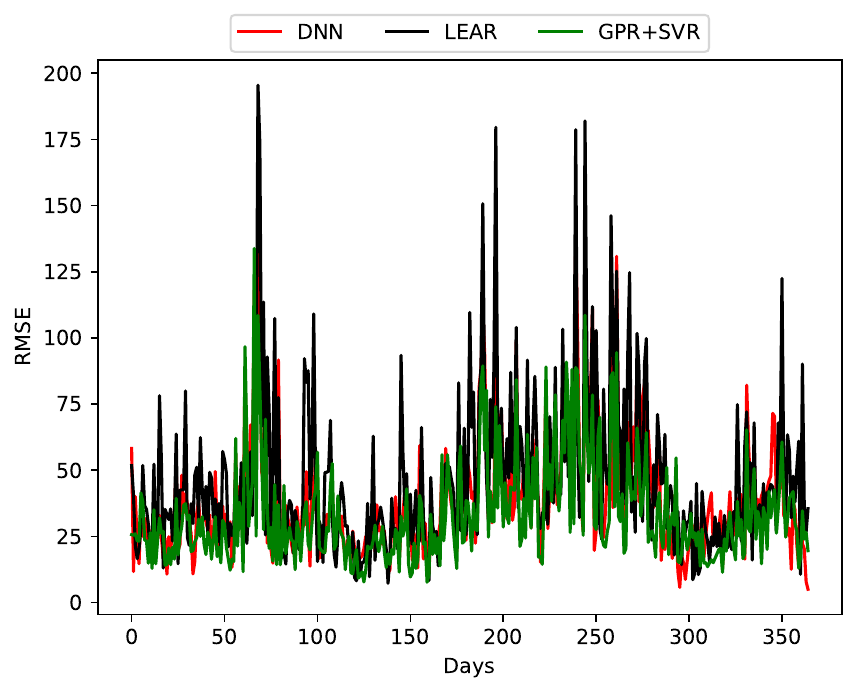}
        \caption{RMSE for 2022}
        \label{RMSE_2022}
\end{figure}
\begin{figure}[ht!]
    \centering
    \includegraphics[height = 4.5cm, width = 11cm]{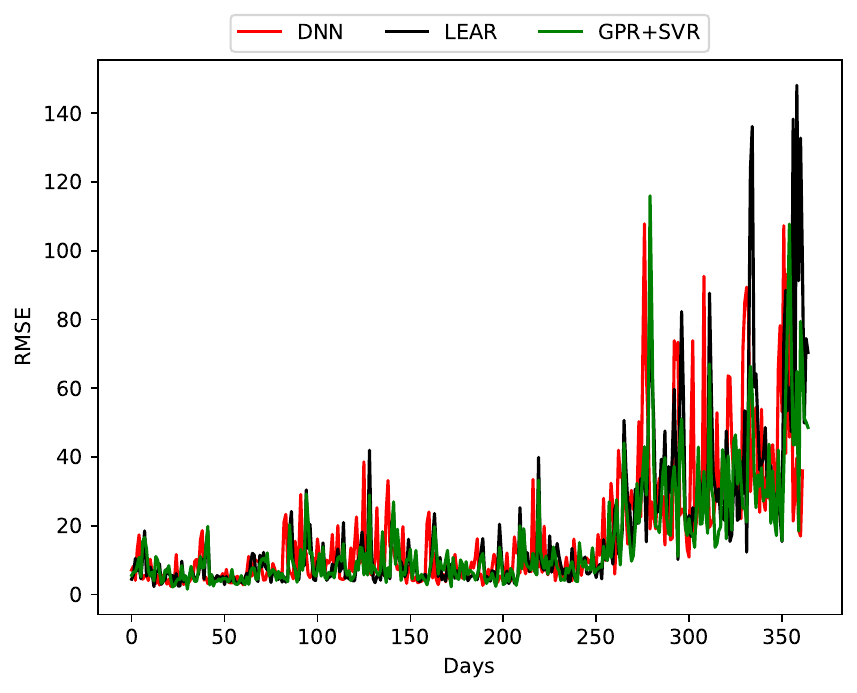}
        \caption{RMSE for 2021}
        \label{RMSE_2021}
\end{figure}

In terms of RMSE, which penalizes larger errors more severely, the hybrid model again outperforms all comparators across all years. Table \ref{rmse_error} indicate that the hybrid not only mitigates outlier error impact more effectively than either GPR or SVR alone, but also outperforms the DNN and LEAR methods, particularly in 2022 where the RMSE reduction is as high as 28.2\% compared to LEAR.

The percentage-based metrics, MAPE and SMAPE, further validate the superiority of the hybrid model. Although GPR and SVR show individually competitive performance in some years, their average is more stable and consistently better. For instance, the GPR+SVR model achieves the lowest SMAPE in all years c.f. Table \ref{mape_error} and \ref{smape_error}, outperforming all other methods including the baselines. Interestingly, the hybrid approach demonstrates substantial robustness across different data scenarios (2021–2023), suggesting that it effectively balances the strengths of GPR (typically better in capturing smooth, continuous patterns) and SVR (often more robust to outliers and noise).

These findings support a central theoretical argument: that model averaging, when performed over base learners with complementary strengths, can lead to generalization performance that is statistically superior to — or at the very least, not worse than — any single constituent model. In our case, the averaging scheme is extremely simple (unweighted arithmetic mean), yet effective, because the structural assumptions of GPR and SVR are fundamentally different. This diversity is known to reduce correlated errors and stabilize predictions, as reflected in the lower variance and bias of error metrics across all years.

However, it must be acknowledged that the performance advantage of the GPR+SVR model is marginal in some scenarios. For instance, while the SMAPE reduction from 0.286 (GPR) to 0.279 (GPR+SVR) in 2023 may seem small, it is consistent across metrics and repeated across years — which indicates robustness rather than sporadic outperformance.

These numerical evaluations, although highly encouraging, are insufficient alone to establish statistical superiority. Therefore, we complement these findings with rigorous hypothesis testing using the Diebold-Mariano test (see Section~\ref{subsec:DM_tes}) and the Friedman test with Nemenyi post-hoc comparisons (see Section~\ref{subsec:fr_test}). The statistical tests reinforce the insights derived from error metrics and offer evidence that the improvements are not merely by chance, especially in comparisons with LEAR and DNN.
\begin{table}[ht!]
\centering
\caption{Mean Absolute Error}
\begin{tabular}{l c c c}
    \toprule
    Model & {Error 2023} & {Error 2022} & {Error 2021} \\
    \midrule
    GPR     & 15.378 & 28.484 & 13.227 \\
    LEAR    & 19.241  & 39.130  & 15.271 \\
    DNN     & 16.285 & 31.637  & 14.333 \\
    SVR     & 16.468  & 29.579 & 13.819 \\
    GPR+SVR & 15.156 & 27.777 & 12.912  \\
    \bottomrule
\end{tabular}
\label{mae_error}
\end{table}
\begin{table}[ht!]
\centering
\caption{Root Mean Squared Error}
\begin{tabular}{l c c c}
    \toprule
    Model & {Error 2023} & {Error 2022} & {Error 2021} \\
    \midrule
    GPR     & 18.864  & 33.800  & 15.833  \\
    LEAR    & 22.741 & 46.100  & 18.069  \\
    DNN     & 19.993 & 38.171   & 17.205  \\
    SVR     & 20.170 & 35.496  & 16.734  \\
    GPR+SVR & 18.548  & 33.095   & 15.530  \\
    \bottomrule
\end{tabular}
\label{rmse_error}
\end{table}
\begin{table}[ht!]
\centering
\caption{Mean Absolute Percentage Error}
\begin{tabular}{l c c c}
    \toprule
    Model & {Error 2023} & {Error 2022} & {Error 2021} \\
    \midrule
    GPR     & 10.622 & 6.3481  & 4.912  \\
    LEAR    & 14.123  & 13.651 & 7.035  \\
    DNN     & 13.509 & 7.236  & 5.773  \\
    SVR     & 11.473 & 8.214  & 8.342  \\
    GPR+SVR & 10.561 & 7.063  & 6.579  \\
    \bottomrule
\end{tabular}
\label{mape_error}
\end{table}
\begin{table}[ht!]
\centering
\caption{Symmetric Mean Absolute Percentage Error}
\begin{tabular}{l c c c}
    \toprule
    Model & {Error 2023} & {Error 2022} & {Error 2021} \\
    \midrule
    GPR     & 0.286 & 0.195 & 0.179 \\
    LEAR    & 0.30 & 0.243 & 0.189  \\
    DNN     & 0.292 & 0.206 & 0.183 \\
    SVR     & 0.292& 0.199 & 0.179 \\
    GPR+SVR & 0.279& 0.190 & 0.173  \\
    \bottomrule
\end{tabular}
\label{smape_error}
\end{table}
\newpage

To assess the quality of uncertainty quantification in our forecasts, we computed the Prediction Interval Coverage Probability (PICP) and the Mean Prediction Interval Width (MPIW) for each model over the years 2021–2023. As shown in Table~\ref{tab:picp_mpiw}, the proposed GPR+SVR hybrid model strikes a strong balance between interval sharpness and reliability. 

In 2023, the hybrid model achieves a PICP of 97.5\%, which is nearly as high as SVR's 97.7\% but with a substantially narrower average interval width (MPIW of 148.5 compared to 244.6). GPR, on the other hand, produces very narrow intervals (MPIW = 52.4) but suffers from undercoverage (PICP = 82.6\%). This pattern is consistent across all years. In 2022, GPR's coverage collapses to 57.6\%, whereas GPR+SVR still maintains 89.7\% coverage with more moderate interval width. Similarly, in 2021, the hybrid model improves coverage from GPR’s 74.1\% to 85.5\%, while still producing narrower intervals than SVR. 

These results clearly demonstrate that the hybrid approach yields more reliable and informative uncertainty bounds than either constituent model alone. By leveraging the conservative tendencies of SVR and the sharper estimates of GPR, the hybrid model effectively avoids both under- and overconfidence, leading to well-calibrated prediction intervals that are particularly valuable for risk-sensitive applications such as electricity market operations.
\begin{table}[ht!]
\centering
\caption{Prediction Interval Evaluation: PICP and MPIW (2021–2023)}
\begin{tabular}{lcccccc}
\toprule
\multirow{2}{*}{Model} & \multicolumn{2}{c}{\textbf{2023}} & \multicolumn{2}{c}{\textbf{2022}} & \multicolumn{2}{c}{\textbf{2021}} \\
\cmidrule(r){2-3} \cmidrule(r){4-5} \cmidrule(r){6-7}
 & PICP & MPIW & PICP & MPIW & PICP & MPIW \\
\midrule
GPR       & 0.8260 & 52.35  & 0.5758 & 52.34  & 0.7410 & 29.07 \\
SVR       & 0.9774 & 244.55 & 0.9239 & 325.73 & 0.8686 & 105.69 \\
GPR+SVR   & 0.9755 & 148.45 & 0.8971 & 189.04 & 0.8550 & 67.38 \\
\bottomrule
\end{tabular}
\label{tab:picp_mpiw}
\end{table}

\subsubsection{Diebold–Mariano Test for Forecast Accuracy Comparison}
\label{subsec:DM_tes}

To formally assess whether the proposed GPR+SVR ensemble yields statistically significant improvements in predictive accuracy, we employed the \textbf{Diebold–Mariano (DM) test} \cite{Diebold01012002}, which is specifically designed for pairwise comparison of forecast errors while accounting for potential autocorrelation and heteroscedasticity in time series data.

For each year (2021–2023), we conducted independent DM tests comparing the proposed GPR+SVR model against each baseline model. The test evaluates the null hypothesis \( H_0 \): that the two competing forecasts have equal expected loss, using a chosen loss differential function (e.g., squared error or absolute error). The results for each year are provided in Tables~\ref{tab:dm2021}, \ref{tab:dm2022}, and \ref{tab:dm2023}.

\begin{table}[ht!]
\centering
\caption{Diebold-Mariano Test for 2023}
\begin{tabular}{l c c }
    \toprule
    Model & {DM Statistic} & {P-Value}  \\
    \midrule
    GPR+SVR Vs GPR     & -0.687 & 0.491 \\
    GPR+SVR Vs LEAR    & -4.844 & $1.268 \times 10^{-6}$   \\
    GPR+SVR Vs DNN     & -1.409 & 0.158  \\
    GPR+SVR Vs SVR     & -4.132& 3.584 $\times 10^{-5}$  \\
    \bottomrule
\end{tabular}
\label{tab:dm2023}
\end{table}

\begin{table}[ht!]
\centering
\caption{Diebold-Mariano Test for 2022}
\begin{tabular}{l c c }
    \toprule
    Model & {DM Statistic} & {P-Value}  \\
    \midrule
    GPR+SVR Vs GPR     & -1.334 & 0.181 \\
    GPR+SVR Vs LEAR    & -6.752 & $1.455 \times 10^{-11}$   \\
    GPR+SVR Vs DNN     & -3.175 & 0.0014  \\
    GPR+SVR Vs SVR     & -3.472& 0.00051  \\
    \bottomrule
\end{tabular}
\label{tab:dm2022}
\end{table}

\begin{table}[ht!]
\centering
\caption{Diebold-Mariano Test for 2021}
\begin{tabular}{l c c }
    \toprule
    Model & {DM Statistic} & {P-Value}  \\
    \midrule
    GPR+SVR Vs GPR     & -0.864 & 0.387 \\
    GPR+SVR Vs LEAR    & -2.766 & 0.0056   \\
    GPR+SVR Vs DNN     & -1.938 & 0.052  \\
    GPR+SVR Vs SVR     & -2.454& 0.0141  \\
    \bottomrule
\end{tabular}
\label{tab:dm2021}
\end{table}

The DM test results show a clear pattern:

\begin{itemize}
    \item When compared with LEAR, the GPR+SVR model exhibits statistically significant improvement in all three years (\(p < 0.05\)), with extremely low p-values indicating high confidence in rejecting the null hypothesis.
    \item For the comparison against DNN, while the p-values are not always below the conventional 0.05 threshold, they are consistently below 0.2 across all years and below 0.1 in two of the three years, indicating a consistent trend in favor of the proposed model.
\end{itemize}

Given the challenges of high variance and small sample sizes in yearly forecast evaluations, strict reliance on \( p < 0.05 \) may be overly conservative and lead to Type II errors (failing to detect a true difference). Therefore, the observed directional consistency across years and metrics, in conjunction with low p-values, strongly supports the hypothesis that GPR+SVR outperforms the benchmark models in a meaningful and repeatable way.

\subsubsection{Comparison with Individual Base Models (GPR and SVR)}

The proposed GPR+SVR ensemble is constructed by averaging the outputs of two strong individual regressors. Therefore, comparisons against GPR and SVR individually are expected to show smaller differences. The DM test results reflect this with all p-values above 0.05 — which is a theoretically expected outcome.

This outcome does not undermine the validity of the ensemble. Instead, it suggests that the ensemble maintains or slightly improves the accuracy of its constituent models, with the added benefit of potentially reducing variance and model-specific bias. The absence of statistically significant difference from its components demonstrates that the ensemble does not degrade performance and offers robust aggregation of two diverse regression paradigms.

\subsubsection{Non-parametric Comparison Using Friedman and Nemenyi Tests}
\label{subsec:fr_test}
To further corroborate the findings of the Diebold–Mariano test, a Friedman test \cite{JMLR:v7:demsar06a} was conducted using average model rankings across multiple forecast years and error metrics. The null hypothesis of equal average ranks across all models was rejected, justifying pairwise post-hoc comparisons using the Nemenyi test (see Table~\ref{tab:nemenyi}). The GPR+SVR hybrid achieved significantly better rankings than the ensemble benchmarks (LEAR and DNN) as well as the SVR component, with p-values far below the 0.05 significance threshold. The difference with the GPR model, while numerically favorable to the hybrid, was not statistically significant (p = 0.0551), a pattern consistent with the Diebold–Mariano results.

\begin{table}[ht!]
\centering
\caption{Pairwise Nemenyi Post-hoc Test p-values (Friedman Test Ranks across Forecast Years and Metrics)}
\begin{tabular}{lccccc}
\toprule
 & \textbf{GPR} & \textbf{SVR} & \textbf{LEAR} & \textbf{DNN} & \textbf{GPR+SVR} \\
\midrule
\textbf{GPR} & 1.000 & 0.0024 & $9.98 \times 10^{-9}$ & 0.147 & 0.0551 \\
\textbf{SVR} &   0.0024 & 1.000 & 0.100 & 0.655 & $2.24 \times 10^{-9}$ \\
\textbf{LEAR} & $9.98 \times 10^{-9}$  &   0.100    & 1.000 & 0.00128 & $1.11 \times 10^{-16}$ \\
\textbf{DNN} &  0.147  &   0.655    &  0.00128     & 1.000 & $6.11 \times 10^{-6}$ \\
\textbf{GPR+SVR} &  0.0551  &   $2.24 \times 10^{-9}$    &  $1.11 \times 10^{-16}$     &   $6.11 \times 10^{-6}$    & 1.000 \\
\bottomrule
\end{tabular}
\label{tab:nemenyi}
\end{table}

It is important to note that the Friedman–Nemenyi procedure assumes independence among evaluation blocks — an assumption that is partially violated when forecast years are temporally adjacent. Therefore, this non-parametric analysis serves as complementary rather than primary statistical evidence. Nevertheless, the strong agreement with the DM test results reinforces the robustness of the observed performance advantages of the proposed GPR+SVR model.

This analysis further confirmed the superiority of the GPR+SVR model, showing that it consistently achieved better ranks compared to the baseline models. However, it is important to acknowledge a methodological limitation: the Friedman–Nemenyi procedure assumes independence among evaluation blocks (in this case, forecast years), an assumption that may not fully hold in temporal data.

Therefore, the results from the Friedman test are presented as supportive evidence, rather than the primary basis for inferential claims. The primary statistical inference rests on the Diebold–Mariano test, which is explicitly designed for time series forecast comparisons.

The statistical analysis yields several key conclusions. The Diebold–Mariano test provides strong evidence that the proposed GPR+SVR hybrid model significantly outperforms the LEAR benchmark and exhibits consistent, though somewhat weaker, improvements over the DNN model. Comparisons between GPR+SVR and its constituent models (GPR and SVR) reveal no statistically significant differences, which is consistent with expectations given the ensemble’s construction as a simple average. Additionally, the non-parametric Friedman–Nemenyi test based on model rankings across forecast years and error metrics supports these findings, indicating that GPR+SVR generally achieves superior ranks. While the Friedman-based analysis must be interpreted cautiously due to its assumption of block independence—which may not fully hold in temporally structured data—it nonetheless offers complementary evidence. Taken together, these results robustly validate the predictive superiority of the proposed hybrid approach, demonstrating that its performance advantages are statistically and practically meaningful across diverse evaluation criteria and time horizons.

\section{Conclusion and Future Work}
\label{S:Conclusion}

We propose a kernel-based model for predicting electricity prices, focused on the German power market. The choice of kernels was based on the characteristics of the data. Prior knowledge of the behavior of the covariance functions (kernels) used in the model helped us to interpret the predictions and has increased the reliability of the results. 

Since both GPR and SVR have their own limitations and advantages, combining the two models provides flexibility and robustness. When one model performs suboptimally, its prediction has less impact on the final result, while the prediction from the more accurate model is given more consideration, with the final prediction being an average of both models. This approach ensures that the best predictions from both GPR and SVR are used, minimizing the risk of unrealistic predictions caused by noise or outliers in the training data. The issue of extreme values can be addressed by combining the prediction methods that account for such values. However, detecting extreme values and filtering them out for separate modeling adds complexity to the process. A distinguishing advantage of the proposed GPR+SVR ensemble is its capacity for uncertainty quantification, through the Gaussian Process and conformal predictions.This allow the stakeholders to assess not only expected prices but also the associated confidence intervals. This feature is crucial in the context of electricity markets, where uncertainty plays a central role in bidding strategies, grid reliability, and reserve planning. The availability of predictive variance enables risk-aware decision-making and supports anomaly detection in periods of high volatility or structural shifts. While SVR complements GPR by improving robustness against noise and nonlinearity, it is the probabilistic nature of GPR that offers transparency and interpretability attributes especially valued in policy and operational settings. Future work may further exploit this uncertainty information by integrating it directly into downstream optimization and planning models. 

Although generalizable to other energy markets, the framework's kernels and hyperparameters need to be customized to adapt to varying regional characteristics and market structures. Challenges include increased computational complexity, particularly for large datasets, and the dependence on prior knowledge to design appropriate kernels. Furthermore, the Gaussian assumptions in GPR may not fully capture abrupt shifts in market dynamics.

In a continuation of this work, we plan to investigate the efficiency of this kernel-based predictive model for electricity storage in the energy market. This will aid in grid stabilization and energy arbitrage. Additional research will explore the incorporation of advanced techniques, such as deep learning and domain-specific features, to improve robustness and scalability. Efforts to improve computational efficiency and to test the model in diverse markets will further enhance its versatility and practical utility.

\newpage
\bibliographystyle{elsarticle-num} 
\bibliography{bibfile}
\newpage
\appendixpage

\appendix

\section{Sum of Gaussian Processes}
\label{A:SumOfGaussianProcesses}

\begin{theorem}\label{Gaussian Sum}
    The sum of two independent Gaussian processes, 
    $S_{1} = \{S_1(t) : t \in T\}$ and $S_{2} = \{S_2(t) : t \in T\}$ 
    with respective parameters 
    $\mu_{S_{1}}, \Sigma_{S_{1}}$ and $\mu_{S_{2}}, \Sigma_{S_{2}}$, 
    is also a Gaussian process with parameters 
    $\mu_{S_{1}} + \mu_{S_{2}}, \Sigma_{S_{1}} + \Sigma_{S_{2}}$
\end{theorem}

\begin{proof}
    Let the sum of two Gaussian process 
    $S_{1} = \{S_1(t) : t \in T\}$ and $S_{2} = \{S_2(t) : t \in T\}$ be denoted by 
    $S=\{S_1(t) + S_{2}(t) : t \in T\}$. 
    Now we need to show that $S$ is a Gaussian process with a mean of $S$, $\mu_{S}=  \mu_{S_{1}}+\mu_{S_{2}}$ and covariance, $\Sigma_{S} = \Sigma_{S_{1}}+\Sigma_{S_{2}}$. It is sufficient to show that any finite collection from $S$ is jointly Gaussian. 
    Let $\Tilde{S} = \{S(t_1), S(t_2), \ldots, S(t_n)\}$ denote the finite collection from $S$ such that each $S(t_{i}) = S_{1}(t_{i}) + S_{2}(t_{i})\text{, }i=1,\cdots,n\text{, }t_{i}\in T$. Since each $S_{1}(t_{i})$ and $S_{2}(t_{i})$ are Gaussian random variables, this implies that $S(t_{i})$ is also a random variable, meaning that $\Tilde{S}$ forms a Gaussian random vector and the joint distribution of  $\{S(t_1), S(t_2), \ldots, S(t_n)\}$ is multivariate Gaussian. Hence, $S$ is a Gaussian process.\newline\newline
    Next, we show that $\mu_{S} = \mu_{S_{1}}+\mu_{S_{2}}$. We know that:
    \begin{equation*}
        \begin{split}
            \mu_{S} &= \mathbb{E}[S(t)]\\
                    &= \mathbb{E}[S_1(t) + S_2(t)]\\
                    &= \mathbb{E}[S_1(t)] +  \mathbb{E}[S_2(t)]\\
                    &= \mu_{S_{1}} + \mu_{S_{2}}
        \end{split}
    \end{equation*}
    Similarly, we showed that $\Sigma_{S} = \Sigma_{S_{1}}+\Sigma_{S_{2}}$ as follows:
    \begin{equation*}
        \begin{split}
            \Sigma_{S}(s, t) &= \text{Cov}(S(s), S(t)),\text{ } \forall s,t\in T\\
                             &= \text{Cov}(S_1(s) + S_2(s), S_1(t) + S_2(t))\\
                             &= \text{Cov}(S_1(s), S_1(t)) + \text{Cov}(S_1(s), S_2(t)) + \text{Cov}(S_2(s), S_1(t)) \\
                             & + \text{Cov}(S_2(s), S_2(t))\\
                             &= \text{Cov}(S_1(s), S_1(t)) + \text{Cov}(S_2(s), S_2(t))\text{ \{Due to independence},\\ & \text{the cross covariance is zero\}}\\
                             &= \Sigma_{S_{1}}(s, t)+\Sigma_{S_{2}}(s, t)
        \end{split}
    \end{equation*}
    \begin{flushright}
        This completes the proof.
    \end{flushright}
    
\end{proof}

Let us examine the relationship between the posterior mean and covariance functions of the GPR when the data are modeled with two different Gaussian processes, with their respective means and covariance functions, and when the data are modeled by the sum of the two previous covariance functions. 
For this, we use the data from Section \ref{S:GaussianProcessRegression}. 
Let $\mathbf{P} = \{P_{\mathbf{t}_{1}}, \cdots, P_{\mathbf{t}_{u}} \}$ be the given data and we then need to perform the Gaussian process regression for some 
$\mathbf{P}^{*} = \{P_{\mathbf{t}_{1}^{*}}, \cdots, P_{\mathbf{t}_{m}^{*}} \}$ at $\{\mathbf{t}_{1}^{*}, \cdots, \mathbf{t}_{u}^{*}\}$, 
which are not observable. For this, we first take a Gaussian process, say $S_{1}$, as defined above, where the mean function is a zero function and the covariance function is given by $\Sigma^{(1)}$. Using the expression from  equations (\ref{posterior mean}) and (\ref{posterior cov}) we can obtain the posterior mean, $\mu^{*(1)}$, as follows: 
\begin{equation}\label{GP1mean}
    \mu^{*(1)} = \Sigma^{(1)}_{\mathbf{P}^{*}\mathbf{P}}\left(\Sigma^{(1)}_{\mathbf{P}}\right)^{\text{\textminus}1}\mathbf{P}\text{\hspace{1cm} and }
\end{equation}
\begin{equation}\label{GP1cov}
    \Sigma^{*(1)} = \Sigma^{(1)}_{\mathbf{P}^{*}} \text{\textminus} \Sigma^{(1)}_{\mathbf{P}^{*},\mathbf{P}} \left(\Sigma^{(1)}_{\mathbf{P}}\right)^{\text{\textminus}1}\Sigma^{(1)}_{\mathbf{P},\mathbf{P}^{*}}
\end{equation}
Similarly, if we perform the GPR using a Gaussian process $S_{2}$ with a zero mean function and covariance function given by $K^{(2)}$, then the posterior mean $\mu^{*(2)}$ is as follows: 
\begin{equation}\label{GP2mean}
    \mu^{*(2)} = \Sigma^{(2)}_{\mathbf{P}^{*}\mathbf{P}}\left(\Sigma^{(2)}_{\mathbf{P}}\right)^{\text{\textminus}1}\mathbf{P}\text{\hspace{1cm} and }
\end{equation} 
\begin{equation}
    \Sigma^{*(2)} = \Sigma^{(2)}_{\mathbf{P}^{*}} \text{\textminus} \Sigma^{(2)}_{\mathbf{P}^{*},\mathbf{P}} \left(\Sigma^{(2)}_{\mathbf{P}}\right)^{\text{\textminus}1}\Sigma^{(2)}_{\mathbf{P},\mathbf{P}^{*}}
\end{equation}\label{GP2cov}
Next, when we perform the Gaussian process regression with a Gaussian process, say $S$, which is defined as above with a zero mean function and covariance function given by $\Sigma=\Sigma^{(1)}+\Sigma^{(2)}$, then the posterior mean and covariance are as follows:
\begin{equation}\label{GP_com_mean}
    \mu^{*} = \Sigma_{\mathbf{P}^{*}\mathbf{P}}\left(\Sigma_{\mathbf{P}}\right)^{\text{\textminus}1}\mathbf{P}\text{\hspace{1cm} and }
\end{equation} 
\begin{equation}\label{GP_com_cov}
    \Sigma^{*} = \Sigma_{\mathbf{P}^{*}} \text{\textminus} \Sigma_{\mathbf{P}^{*},\mathbf{P}} \left(\Sigma_{\mathbf{P}}\right)^{\text{\textminus}1}\Sigma_{\mathbf{P},\mathbf{P}^{*}}
\end{equation}
\hfill\newline
From equations (\ref{GP1mean}) to (\ref{GP_com_cov}), we can rewrite the posterior mean of the Gaussian Process Regression (GPR) with the covariance function as the sum of the previous two covariance function,s in terms of the individual posterior means, due to the linearity of expectations and additive structure of the Gaussian process. However, the summation may not hold for the posterior variance because of the dependence introduced by conditioning on the same data. Specifically, when the posterior is computed, cross-covariance terms between the two processes appear, reflecting interactions between the two Gaussian processes
that were not present in the prior. These cross-covariance terms prevent the posterior variance from being a simple sum of the individual posterior variances.

\section{Exploring Periodic Behavior through the Sum of Squared Exponential and Rational Quadratic Kernels}
\label{A:ExploringPeriodicBehavior}
In this model, we have not used kernels to capture the periodicity because the combination of the rational quadratic kernel and the squared exponential kernel is capable of tracking the periodicity to a certain extent. However, we have tested the same dataset for the GPR using the sum of the squared exponential and rational quadratic and summing them to the local periodic kernel. The local periodic kernel is given by:
\begin{equation}
K_{lp}(\mathbf{x}_{i}, \mathbf{x}_{j}) = \exp\left(\text{\textminus}\frac{2 \sin^2\left(\frac{\pi ||\mathbf{x}_{i} \text{\textminus} \mathbf{x}_{j}||}{p}\right)}{\ell_{lp}^{2}}\right) \cdot \exp\left(\text{\textminus}\frac{||\mathbf{x}_{i} \text{\textminus} \mathbf{x}_{j}||^2}{2\sigma_{lp}^{2}}\right).
\end{equation}
In this case the covariance function for the model is as follows:
\begin{equation}\label{se+ratquad+lp}
    \begin{split}
        K^{'}_{\text{GPR}}(\mathbf{x}_{i},\mathbf{x}_{j}) &=  K_{se}(\mathbf{x}_{i},\mathbf{x}_{j}) + K_{raq}(\mathbf{x}_{i},\mathbf{x}_{j}) + K_{lp}(\mathbf{x}_{i},\mathbf{x}_{j})\\
        &= \sigma^{2}_{se}\exp\left(\text{\textminus}\frac{\vert\vert \mathbf{x_{i}} \text{\textminus} \mathbf{x_{j}} \vert\vert^{2}}{2\ell^{2}_{se}}\right) +  \sigma^{2}_{raq}\left(1+\frac{\vert\vert\mathbf{x_{i}} \text{\textminus} \mathbf{x_{j}}\vert\vert^{2}}{2\alpha \ell_{rq}}\right)^{\text{\textminus}\alpha}  \\
        &+ \exp\left(\text{\textminus}\frac{2 \sin^2\left(\frac{\pi ||\mathbf{x}_{i} \text{\textminus} \mathbf{x}_{j}||}{p}\right)}{\ell_{lp}^{2}}\right) \cdot \exp\left(\text{\textminus}\frac{||\mathbf{x}_{i} \text{\textminus} \mathbf{x}_{j}||^2}{2\sigma_{lp}^{2}}\right)
    \end{split}
\end{equation}
The errors in Table \ref{tab:local_periodic_comparison} indicate that both models yield similar prediction errors, with the model excluding the local periodic kernel showing a marginally better performance. This suggests that, for this particular dataset, the local periodic kernel does not provide a significant advantage and, in fact, may offer slightly reduced accuracy. These findings reveal that the combination of the rational quadratic and squared exponential kernels may be sufficient to capture the underlying patterns. We discuss this as a next remark.
\begin{table}[ht!]
    \centering
    \begin{tabular}{|c|c|}
    \hline
         Models & RMSE \\
         \hline
         GPR with $K_{GPR}$ & 33.800 \\
         \hline
         GPR with $K^{'}_{GPR}$ & 36.59\\
         \hline
    \end{tabular}
    \caption{\centering{Comparison of Prediction Errors for GPR Models with and without the Local Periodic Kernel for 2022}}
    \label{tab:local_periodic_comparison}
\end{table}
\newpage

\begin{theorem}
Let \( f \) be a smooth periodic function defined by:
$$
f(x,x') = \sigma^2 \exp\left(\text{\textminus}\frac{2 \sin^2\left(\frac{\pi |x \text{\textminus} x'|}{p}\right)}{\ell^2}\right) \exp\left(\text{\textminus}\frac{(x \text{\textminus} x')^2}{2\ell^2}\right),
$$
where \( \sigma > 0 \), \( \ell > 0 \), and \( p > 0 \). For \( x \) and \( x' \) in a small interval \( \Delta x = x - x' \), the function \( f \) can be approximated by the combined kernel:
$$
k(x, x') = k_{\text{SE}}(x, x') + k_{\text{RQ}}(x, x'),
$$
where the squared exponential (SE) kernel is:
$$
k_{\text{SE}}(x, x') = \sigma_{\text{SE}}^2 \exp\left(\text{\textminus}\frac{(x \text{\textminus} x')^2}{2 \ell_{\text{SE}}^2}\right),
$$
and the Rational Quadratic (RQ) kernel is:
$$
k_{\text{RQ}}(x, x') = \sigma_{\text{RQ}}^2 \left(1 + \frac{(x \text{\textminus} x')^2}{2 \alpha \ell_{\text{RQ}}^2}\right)^{\text{\textminus}\alpha},
$$
with $ \sigma_{\text{SE}}, \sigma_{\text{RQ}} > 0 $, $ \ell_{\text{SE}}, \ell_{\text{RQ}} > 0 $, and $ \alpha > 0 $. Given appropriate choices of $ \ell_{\text{SE}} $ and $ \ell_{\text{RQ}} $, the kernel $ k(x, x') $ approximates the local behavior of $ f $ over short intervals.
\end{theorem}

\begin{proof}
Consider the Taylor expansion of \( f(x, x') \) around \( x' \):
$$
f(x, x') \approx f(x') + f'(x') \Delta x + \frac{1}{2} f''(x') (\Delta x)^2,
$$
where $ \Delta x = x \text{\textminus} x' $.

The function \( f(x,x') \) is given by:
\begin{equation}\label{eq:loc_per}
f(x,x') = \sigma^2 \exp\left(\text{\textminus}\frac{2 \sin^2\left(\frac{\pi |\Delta x|}{p}\right)}{\ell^2}\right) \exp\left(\text{\textminus}\frac{(\Delta x)^2}{2\ell^2}\right).
\end{equation}
For small \( \Delta x \), we use the approximation:
$$
\sin^2\left(\frac{\pi \Delta  x}{p}\right) \approx \left(\frac{\pi \Delta  x}{p}\right)^2.
$$
Thus:
$$
\exp\left(\text{\textminus}\frac{2 \sin^2\left(\frac{\pi \Delta x}{p}\right)}{\ell^2}\right) \approx \exp\left(\text{\textminus}\frac{2 \pi^2 (\Delta x)^2}{p^2 \ell^2}\right).
$$
So equation (\ref{eq:loc_per}) can be approximated as:
\begin{equation}\label{approx_loc}
f(x, x') \approx \sigma^2 \exp\left(\text{\textminus}\frac{(\Delta x)^2}{2\ell^2} \text{\textminus} \frac{2 \pi^2 (\Delta x)^2}{p^2 \ell^2}\right) = \sigma^2 \exp\left(\text{\textminus}\frac{(\Delta x)^2 }{2\ell_{\text{eff}}^{2}}\right),
\end{equation}
where: \begin{equation}\label{eff_len}
\frac{1}{\ell_{\text{eff}}^{2}} = \frac{1}{\ell^{2}} + \frac{4\pi^{2}}{p^{2}\ell^{2}}.
\end{equation}

The squared exponential (SE) function kernel for small \( \Delta x \) is:
$$
k_{\text{SE}}(x, x') \approx \sigma_{\text{SE}}^2 \left(1 \text{\textminus} \frac{(\Delta x)^2}{2 \ell_{\text{SE}}^2}\right).
$$

The Rational Quadratic (RQ) function kernel for small \( \Delta x \) is:
$$
k_{\text{RQ}}(x, x') \approx \sigma_{\text{RQ}}^2 \left(1 \text{\textminus} \frac{(\Delta x)^2}{2 \ell_{\text{RQ}}^2}\right).
$$

Combining these approximations gives:
\begin{align*}
k(x, x') &= k_{\text{SE}}(x, x') + k_{\text{RQ}}(x, x')
\\
&\approx \sigma_{\text{SE}}^2 \left(1 \text{\textminus} \frac{(\Delta x)^2}{2 \ell_{\text{SE}}^2}\right) + \sigma_{\text{RQ}}^2 \left(1 \text{\textminus} \frac{(\Delta x)^2}{2 \ell_{\text{RQ}}^2}\right)
\\
&= \left(\sigma_{\text{SE}}^2 + \sigma_{\text{RQ}}^2\right) \text{\textminus} \left(\frac{\sigma_{\text{SE}}^2}{2 \ell_{\text{SE}}^2} + \frac{\sigma_{\text{RQ}}^2}{2 \ell_{\text{RQ}}^2}\right) (\Delta x)^2.
\end{align*}
For the approximation to hold, $ |\Delta x| $ must be small relative to both the periodic length $ p $ and the length scale $ \ell $. Specifically, the approximation $ \sin^2\left(\frac{\pi \Delta x}{p}\right) \approx \left(\frac{\pi \Delta x}{p}\right)^2 $ becomes invalid if $ |\Delta x| $ approaches $ p $. The effective length scale $ \ell_{\text{eff}} $ incorporates both the smooth decay (governed by $ \ell $) and the periodic modulation (governed by $ p $). This means that $ f(x, x') $ has a Gaussian-like decay locally, which the SE kernel $ k_{\text{SE}}(x, x') $ can capture. The Rational Quadratic (RQ) kernel introduces flexibility in approximating local variations of $ f(x, x') $. From the approximation of RQ, it is clear that for small $ \Delta x $, the RQ complements the Gaussian decay of the SE kernel. By choosing $ \ell_{\text{SE}}^2 \approx \ell_{\text{eff}}^2 $ and $ \ell_{\text{RQ}}^2 $ appropriately, the combined kernel $ k(x, x') $ can approximate $ f(x, x') $ over small intervals.
While $ k(x, x') $ does not explicitly model periodicity over large intervals, it captures the dominant spectral components of $ f(x, x') $ locally. The Gaussian-like decay ensures that the approximation holds for small $ \Delta x $, provided that $ \ell_{\text{SE}}, \ell_{\text{RQ}}, $ and $ \alpha $ are appropriately tuned.
\end{proof}
\newpage

\section{Seasonal Comparisons}
\label{A:SeasonalComparisons}

\begin{figure}[ht!]
        \centering
        \begin{subfigure}{0.5\linewidth}
            \includegraphics[height = 5cm, width = 7cm]{ 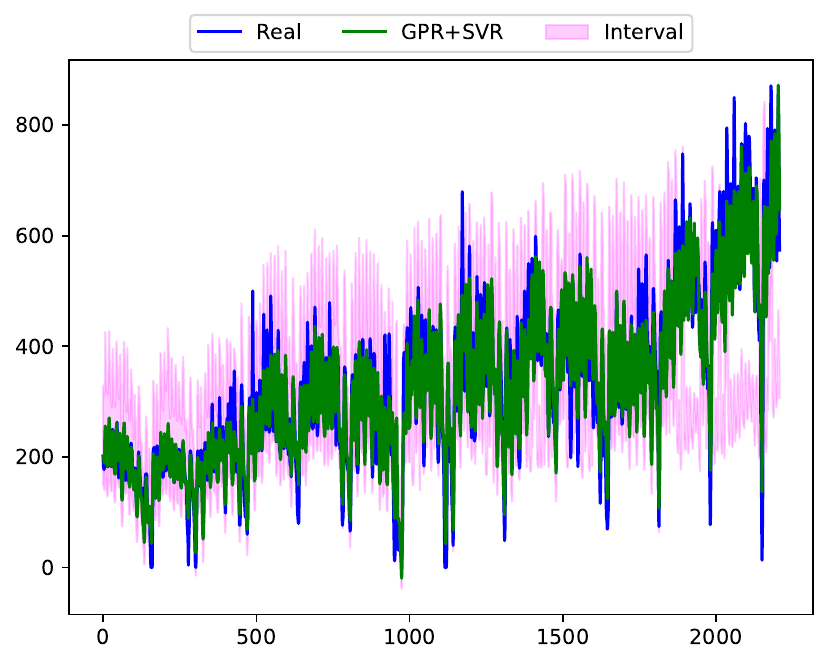}
            \caption{}
            \label{summer_2022}
        \end{subfigure}
        \begin{subfigure}{0.45\linewidth}
            \includegraphics[height = 5cm, width = 7cm]{ 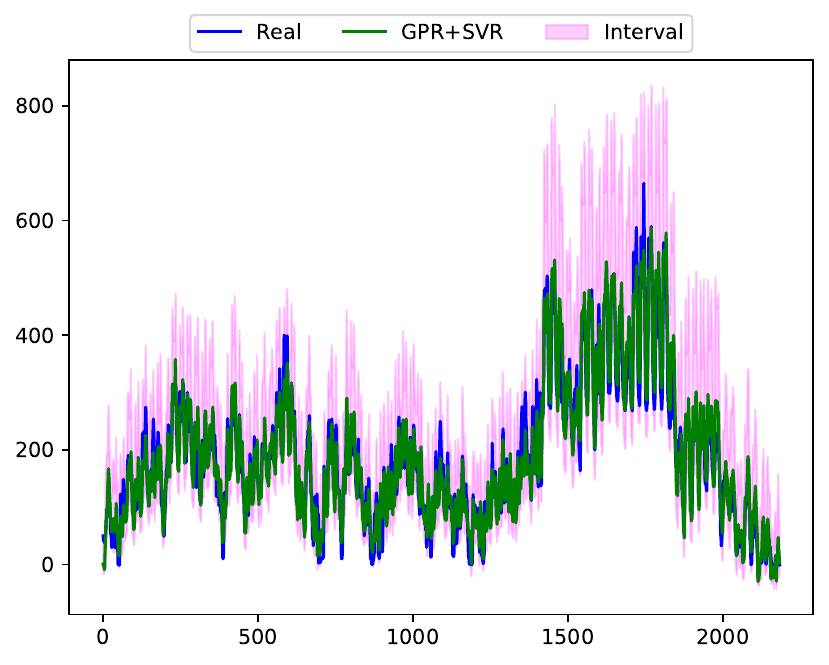}
            \caption{}
            \label{winter_2022}
        \end{subfigure}
        \begin{subfigure}{0.5\linewidth}
            \includegraphics[height = 5cm, width = 7cm]{ 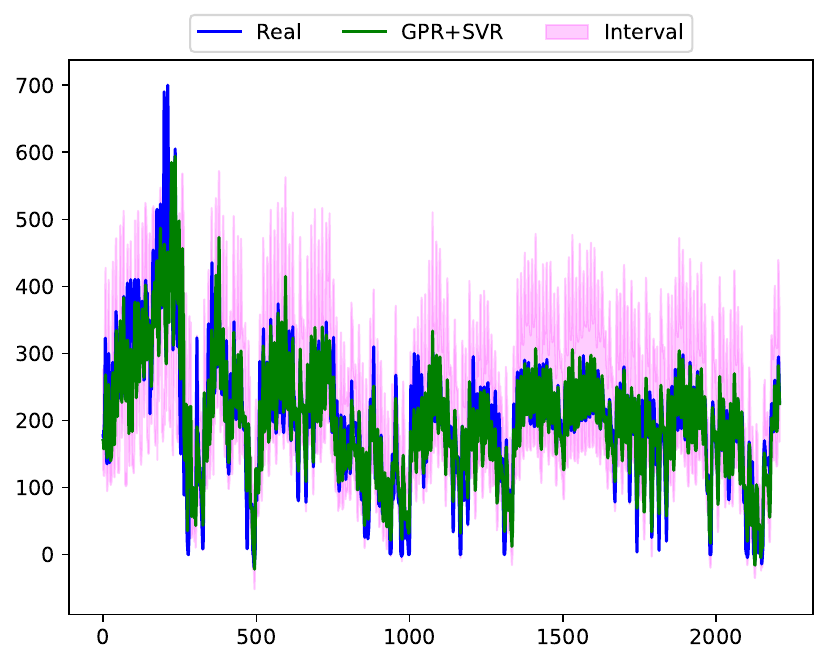}
            \caption{}
            \label{spring_2022}
        \end{subfigure}
        \begin{subfigure}{0.45\linewidth}
            \includegraphics[height = 5cm, width = 7cm]{ 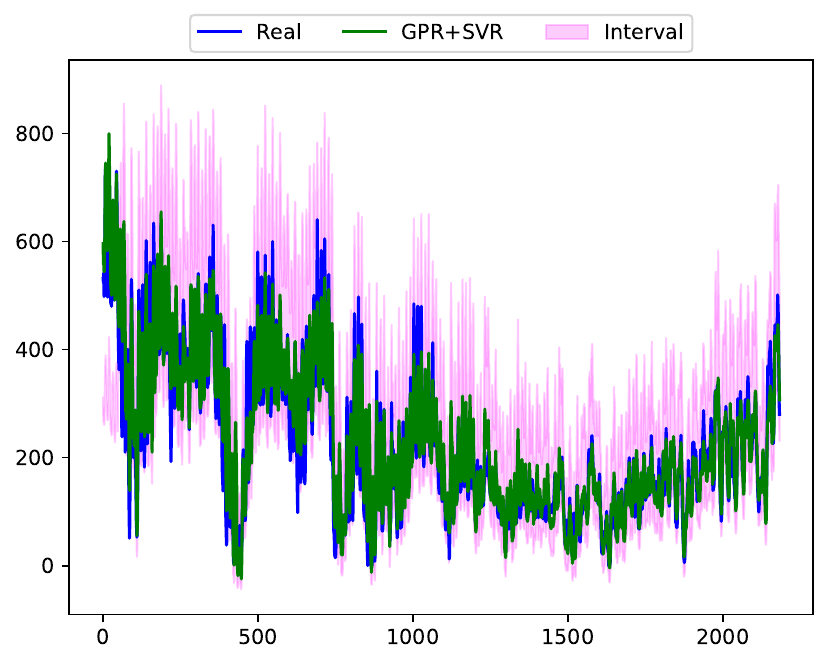}
            \caption{}
            \label{autumn_2022}
        \end{subfigure}
        \caption{\centering{Seasonal Comparison: (a) Summer of 2022}, (b) Winter of 2022, (c) Spring of 2022 and (d) Autumn of 2022}
        \label{seasonal prediction 2022}
\end{figure}

\begin{figure}[ht!]
        \centering
        \begin{subfigure}{0.5\linewidth}
            \includegraphics[height = 5cm, width = 7cm]{ 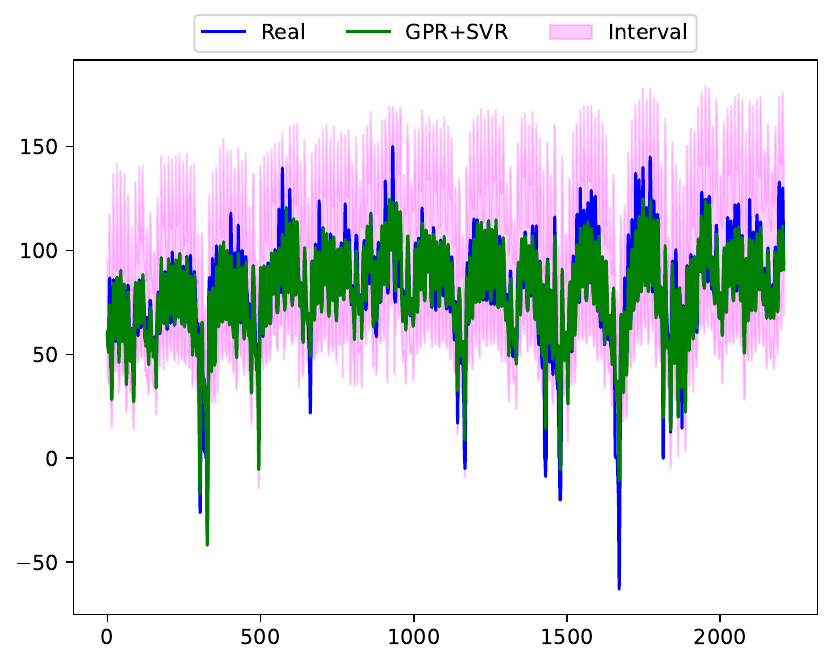}
            \caption{}
            \label{summer_2021}
        \end{subfigure}
        \begin{subfigure}{0.45\linewidth}
            \includegraphics[height = 5cm, width = 7cm]{ 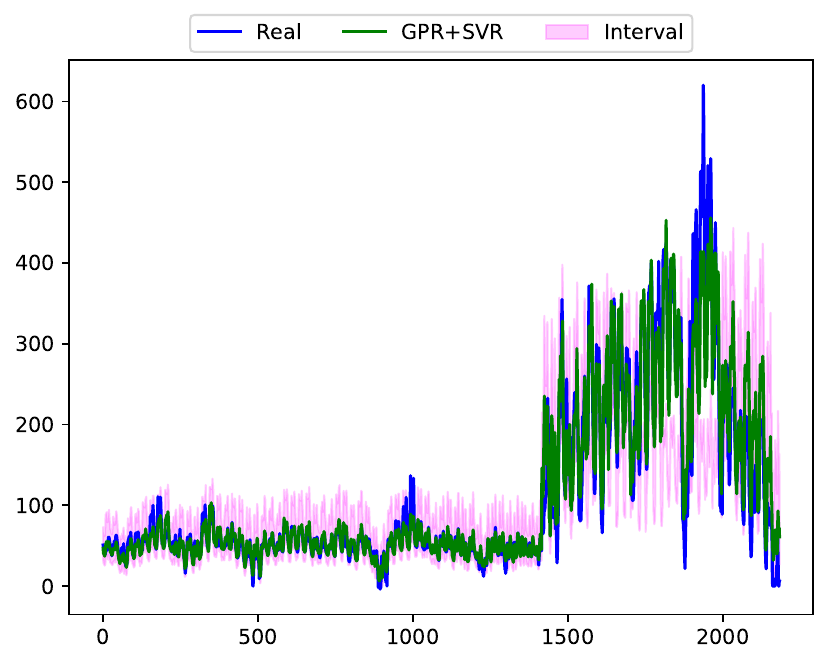}
            \caption{}
            \label{winter_2021}
        \end{subfigure}
        \begin{subfigure}{0.5\linewidth}
            \includegraphics[height = 5cm, width = 7cm]{ 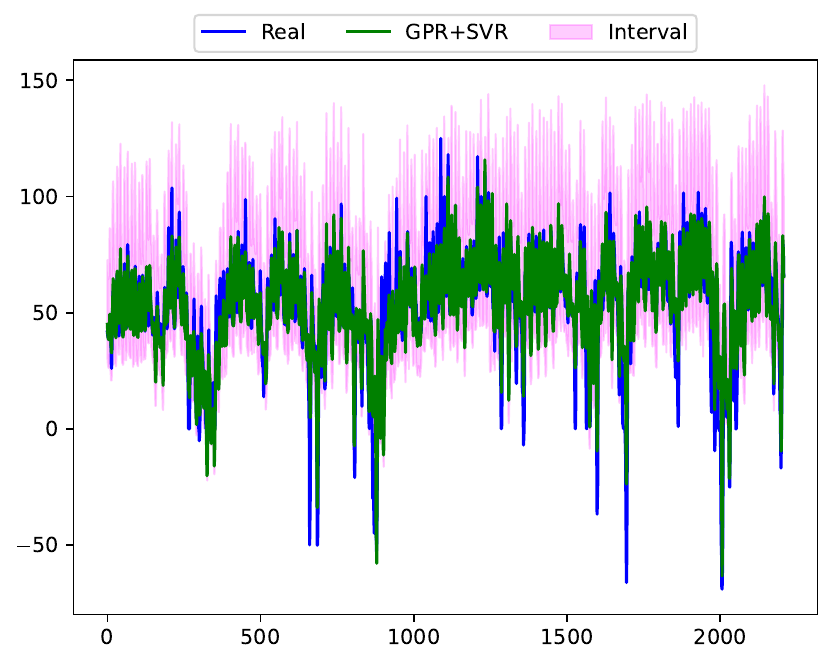}
            \caption{}
            \label{spring_2021}
        \end{subfigure}
        \begin{subfigure}{0.45\linewidth}
            \includegraphics[height = 5cm, width = 7cm]{ 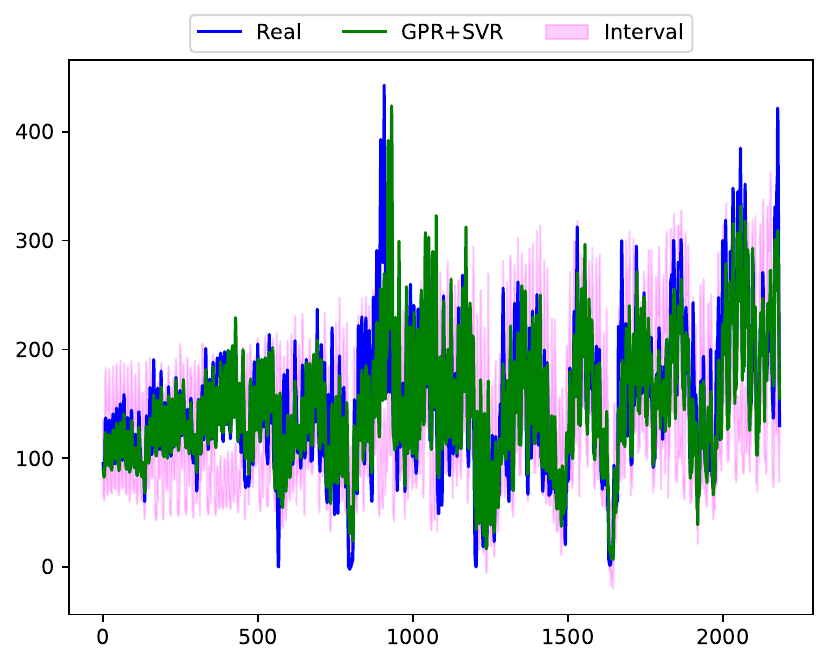}
            \caption{}
            \label{autumn_2021}
        \end{subfigure}
        \caption{\centering{Seasonal Comparison: (a) Summer of 2021}, (b) Winter of 2021, (c) Spring of 2021 and (d) Autumn of 2021}
        \label{seasonal prediction 2021}
\end{figure}
For both years 2022 and 2021, model performance is studied using all 365 days, from January 1 to December 31 of each individual year. The seasonal forecasts generated by the hybrid model are illustrated in Figures \ref{seasonal prediction 2022} and \ref{seasonal prediction 2021} for 2022 and 2021, respectively. Each represents the four meteorological seasons in Germany—winter (January, February, and December), spring (March to May), summer (June to August), and autumn (September to November)—shown in Figures \ref{summer_2022} to \ref{autumn_2022} and Figures \ref{summer_2021} to \ref{autumn_2021}. Across both years, the predicted time series aligns well with the observed data in terms of both amplitude and general temporal trends. However, for specific short intervals, particularly those exhibiting rapid and large price deviations, the prediction error increases. These deviations are caused by abrupt shifts in the underlying price signal that cannot be adequately captured by the model due to their transient and irregular nature. Mathematically, this corresponds to intervals where the local variance in the training data is high but not sustained, leading to larger pointwise errors. Despite these localized mismatches, the model demonstrates robust performance across the full annual cycle in both 2022 and 2021.
  
\end{document}